\documentclass[10pt,journal,compsoc]{IEEEtran}

\usepackage[english]{babel}
\usepackage{amsmath,amsthm,amssymb,amsfonts}
\usepackage{graphicx, array}
\usepackage{hyperref}
\hypersetup{colorlinks,linkcolor={blue},citecolor={blue},urlcolor={red}}  
\usepackage{euler, latexsym, epsfig,epic}
\usepackage{tikz}

\makeatletter
\let\@fnsymbol\@arabic
\makeatother

\theoremstyle{plain}
\newtheorem{thm}{Theorem}
\newtheorem{prp}{Proposition}

\newtheorem{lem}{Lemma}
\newtheorem{remark}{Remark}

\newtheorem{ex}{Example}

\newtheorem{form}{Formulation}

\theoremstyle{definition}
\newtheorem{definition}{Definition}

\newcommand{\myparagraph}[1]{\smallskip\noindent\textbf{#1.}}

\def\0{\boldsymbol{0}}

\def\Y{\boldsymbol{Y}}

\def\C{\mathcal{C}}

\def\U{\boldsymbol{U}}

\def\H{\mathcal{H}}

\def\P{\boldsymbol{P}}

\def\I{\mathcal{I}}

\def\J{\mathcal{J}}

\def\K{\mathcal{K}}

\def\M{\mathcal{M}}

\def\P{\mathcal{P}}
\def\PP{\mathbb{P}}
\def\U{\mathscr{U}}
\def\sV{\mathscr{V}}

\def\V{\mathcal{V}}

\def\Y{\mathcal{Y}}

\DeclareMathOperator{\rank}{rank}

\DeclareMathOperator*{\Gr}{Gr}

\def\RR{\mathbb{R}}
\def\CC{\mathbb{C}}

\long\def\answer#1{}

\long\def\comment#1{}

\begin{document} 

\title{Low-Rank Matrix Completion Theory \\ via Pl\"ucker Coordinates}

\author{
    \IEEEauthorblockN{Manolis C. Tsakiris \\ 
    \vspace{0.2in}   
    \vspace{0.05in}         
    \IEEEcompsocitemizethanks{\IEEEcompsocthanksitem Key Laboratory of Mathematics Mechanization, Academy of Mathematics and Systems Science, Chinese Academy of Sciences, Beijing, 100190, China.
    \IEEEcompsocthanksitem manolis@amss.ac.cn}
} \normalsize}

\IEEEtitleabstractindextext{
\begin{abstract}
Despite the popularity of low-rank matrix completion, the majority of its theory has been developed under the assumption of random observation patterns, whereas very little is known about the practically relevant case of non-random patterns. Specifically, a fundamental yet largely open question is to describe patterns that allow for unique or finitely many completions. This paper provides three such families of patterns for any rank and any matrix size. A key to achieving this is a novel formulation of low-rank matrix completion in terms of Pl\"ucker coordinates, the latter a traditional tool in computer vision. This connection is of potential significance to a wide family of matrix and subspace learning problems with incomplete data. 
\end{abstract}
}

\maketitle

%%%%%%%%
\section{Introduction} \label{section:Introduction}

\subsection{Low-Rank Matrix Completion (LRMC)} \label{subsection:LRMC}
Low-Rank Matrix Completion (LRMC) is a fundamental interdisciplinary problem, in which one seeks to complete an $m \times n$ real matrix $X$ of rank $r \ll \min\{m,n\}$ by only observing a subset of its entries. LRMC has many applications such as photometric stereo \cite{wu2010robust}, depth enhancement \cite{lu2014depth} and reflection removal \cite{han2017reflection} in computer vision, recommendation systems \cite{bennett2007netflix}, time-series forecasting \cite{gillard2018structured}, localization in IoT networks \cite{nguyen2019localization}, seismic data reconstruction \cite{ma2013three} and many others.  

There are at least three factors that make LRMC ubiquitous. First, high-dimensional data have strong underlying patterns, which manifest via a low-intrinsic dimension with respect to some geometric model. The most popular among such models are linear subspaces, often leading to the phenomenon of high-dimensional data matrices approximately having low-rank. Second, such data often occur in incomplete form, as typically not all features are observed across all samples. For example, many times it is too expensive or even infeasible to directly acquire the entire target data matrix; rather, one samples a small part of it at specially designed locations and then tries to infer the missing values. This is known as \emph{active matrix completion} \cite{chakraborty2013active} or as \emph{matrix completion with queries} \cite{ruchansky2015matrix}. Third, the otherwise NP-hard LRMC problem is provably solvable under mild conditions via convex optimization in polynomial time \cite{candes2009exact}, and scalable implementations \cite{jain2013low,mu2016scalable} have made LRMC a powerful tool in applications. The popularity of LRMC has also led to its serving as a fertile paradigm in machine learning theory, e.g., in the study of data-induced benign structures in non-convex optimization \cite{ge2016matrix,balcan2019non} or for understanding the behavior of gradient descent in deep networks \cite{arora2019implicit}. 

%%%%
\subsection{Random Patterns} \label{subsection:Random Patterns}

Let $\M(r, m \times n)$ be the set of $m \times n$ real matrices of rank at most $r$. In their seminal paper \cite{candes2009exact} Cand\`es \& Recht established that an incoherent matrix $X \in \M(r, m \times n)$ can be correctly completed with high probability via convex nuclear norm minimization from $\mathcal{O}(r \max\{m,n\}^{6/5} \log \max\{m,n\})$ uniformly at random observed entries. Ignoring logarithms, this sampling rate is almost optimal since $\M(r, m \times n)$ is an algebraic variety of dimension $r(m+n-r)$ \cite{harris2013algebraic} and sampling fewer entries gives infinitely many completions. Reaching an optimal sampling rate by removing the $6/5$ exponent was achieved shortly after in another famous paper by Cand\`es \& Tao \cite{candes2010power}, who employed a remarkable chain of long and complicated combinatorial arguments. Then Recht \cite{recht2011simpler} simplified the proof by using results of Gross in quantum information theory, who in turn generalized the theorem to arbitrary linear measurements of $X$ \cite{gross2011recovering}. Since then, a large literature on LRMC has been developed with the majority of theoretical analyses adopting a random observation pattern. 

%%%%
\subsection{Deterministic Patterns} \label{subsection:Deterministic Patterns}

In contrast to random patterns, deterministic observation patterns have received much less attention, and have been far from being understood in the context of LRMC. Among the few works that have considered them are\footnote{For deterministic patterns in seeking a minimal-rank completion of a generic incomplete matrix see \cite{shapiro2018matrix,bernstein2020typical}. For deterministic patterns for matrix completion in a union of subspaces see \cite{ongie2021tensor}.} \cite{singer2010uniqueness,kiraly2012combinatorial,lee2013matrix,chen2015completing,kiraly2015algebraic,pimentel2016characterization,bernstein2017completion,liu2019matrix,chatterjee2020deterministic,foucart2020weighted}. This is an important gap in the literature, because in practice the pattern of missing entries depends on the specific nature of the application and thus can be very non-random. For example, in computer vision the missing entries occur due to occlusions or shadows in the scene; in recommendation systems they occur as a consequence of the existing preferences of different users. 

\myparagraph{Uniquely Completable Patterns} Given an observation pattern $\Omega$, which we can think of as an $m \times n$ matrix of $0$'s and $1$'s, one would like to know whether observing a sufficiently generic matrix $X \in \M(r, m \times n)$ at the locations of $1$'s uniquely defines $X$, in the sense that $X$ is the only generic rank-$r$ matrix agreeing with the incomplete matrix; here \emph{generic} means that the matrix avoids certain pathologies in alignment with the nature of the application at hand. If the answer is \emph{yes}, then one would already know that any algorithm that produces a generic rank-$r$ completion, produces the ground-truth matrix $X$. If on the other hand the answer is \emph{no}, then one would know that either i) more properties of $X$ must be taken into account and specialized algorithms leveraging these properties must be employed or ii) one needs to sample more entries from $X$. Conversely, if a characterization is known for the patterns that allow unique completion, then one would seek to sample data according to these patterns, which is important when one designs queries for active matrix completion \cite{chakraborty2013active,ruchansky2015matrix}.   

\myparagraph{Finitely Completable Patterns} A less intuitive but even more fundamental question than unique completability is whether for a given pattern $\Omega$, a sufficiently generic matrix $X \in \M(r,m \times n)$ observed at $\Omega$ has finitely many rank-$r$ completions. Finite completability is necessary for unique completability, and thus understanding the finitely completable patterns is key to understanding the uniquely completable ones. Moreover, finitely completable patterns are important in their own right, because in principle one may extract the finitely many completions or a subset of them and then use other algorithms or tests to isolate a completion that best suits the application. 

\subsection{Related Work} Among early works that concretely and formally posed the problem of characterizing the uniquely or finitely completable patterns in LRMC are \cite{singer2010uniqueness,kiraly2012combinatorial,kiraly2015algebraic}. This characterization was settled in \cite{singer2010uniqueness} for $r=1$ using rigidity theory. Important connections with graph theory, combinatorics and algebraic geometry were established in \cite{kiraly2012combinatorial,kiraly2015algebraic}. Nevertheless, the question of \emph{what exactly is the nature of these patterns}, or in other words \emph{how do finitely or uniquely completable patterns look like}, as for example one would need to know for designing queries, still remained open. This was undertook by \cite{pimentel2016characterization}, where for the special case of $m \times r(m-r)$ rank-$r$ matrices, a characterization of finitely completable patterns was suggested using algebraic geometry. Several important insights emerged from this work; however, their algebraic geometric analysis needs to be revisited \footnote{Details about this issue are given in \S \ref{section:Alarcon}.}. Using tropical geometry, \cite{bernstein2017completion} described the finitely completable patterns for $r=2$. Finally, an asymptotic in the matrix dimensions $m,n$ characterization of approximately uniquely completable patterns was given in \cite{chatterjee2020deterministic}. Overall, for $r>2$ very little else seems to be known about uniquely or finitely completable patterns.

\subsection{The Role of Algebraic Geometry}  The problem of unique or finite completability in LRMC is in its nature a problem of algebraic geometry, the field of mathematics concerned with solutions to polynomial systems of equations; interestingly, LRMC itself can be used to solve such equations \cite{ma2012computing}. One of the most successful applications of algebraic geometry has been computer vision. Indeed, state-of-the-art commercial structure from motion systems often rely on the \emph{5-point algorithm} \cite{nister2004efficient}, a polynomial system solver for recovering the relative pose of two calibrated cameras from five point correspondences. Moreover, using algebraic geometry, a lot has been achieved in computer vision theory in latest years in the theory of multi-view geometry; for example see \cite{agarwal2019ideals} and references therein. Algebraic geometry has also played an important role in generalized principal component analysis \cite{Vidal:PAMI05} or else known as subspace clustering \cite{vidal2011subspace}, as well as in phase retrieval \cite{conca2015algebraic}. 
This is not a coincidence: low-rank matrices, unions of subspaces and Grassmannians are prevalent mathematical objects in machine learning with inherent algebraic-geometric character.

\subsection{Paper Contributions} This paper takes several steps forward in the theory of LRMC by contributing the following:

\begin{enumerate}
\item For any rank $r$ and any matrix dimensions $m$ and $n$, a family of $m \times n$ finitely completable patterns (\S \ref{section:main}). 
\item For any $r,m,n$, two families of $m \times n$ uniquely completable patterns (\S \ref{section:main}). 
\item A potential connection with probabilistic methods within LRMC, since in the low-rank regime the families of 1) and 2) occur with high probability under uniform sampling (\S \ref{section:main}).
\item A novel formulation of LRMC in terms of Pl\"ucker coordinates (\S \ref{section:MC-PC}), the latter a traditional tool in computer vision. Besides a key technical feature in the proofs of 1) and 2), this connection is also of potential significance to a wider class of matrix and subspace learning problems from incomplete data. 
\item A discussion of how existing flaws in the literature are fixed by 1), 2) and 4) (\S \ref{section:Alarcon}). 
\item An accessible technical development for readers not familiar with algebraic geometry, with technicalities suppressed as much as possible in favor of intuition and examples (\S \ref{section:main}-\S\ref{section:proofs}).
\end{enumerate}

The rest of the paper is organized as follows. \S \ref{section:main} discusses the main results, \S \ref{section:AG} gives the necessary algebraic geometry background, \S \ref{section:MC-PC} formulates LRMC as a problem on the Grassmannian using Pl\"ucker coordinates, \S \ref{section:Alarcon} discusses connections with literature and \S \ref{section:proofs} gives proofs.

%%%%%%%%%%%%%%%%%

\section{Main Results} \label{section:main}

\subsection{Notation} Before we state the main results of this paper, we need a little preparation. For a positive integer $k$ we let $[k]=\{1,\dots,k\}$. We can think of an observation pattern $\Omega$ as an $m \times n$ matrix of $0$'s and $1$'s. Alternatively, we can think of it as a subset $\Omega \subset [m] \times [n]$ with every element of $\Omega$ indicating an observed location in the matrix. It will further be convenient to think of $\Omega$ column-wise, in which case we write $\Omega = \bigcup_{j \in [n]} \omega_j \times \{j\} \subset [m] \times [n]$; here $\omega_j$ is the subset of $[m]$ that contains the observed row-indices across column $j$. The next example illustrates these conventions:

\begin{ex} \label{example:main}
An $\Omega \subset [6] \times [5]$ with $r=2, \, m=6, \, n=5$, depicted both in matrix indicator form and set-theoretic form:
$$\Omega = \begin{bmatrix} 
1 & 0 & 0 & 1 & 1 \\
1& 0 & 1 & 1 & 0 \\
1& 0 & 0& 0 & 1 \\
1 & 1 & 1 & 1 & 0 \\
1 & 1 & 0 & 1 & 1\\
0 & 1 & 0 & 1 & 0
\end{bmatrix}$$
\begin{align}
\cong & \big(\underbrace{\{1,2,3,4,5\}}_{\omega_1} \times \{1\}\big) \cup  \big(\underbrace{\{4,5,6\}}_{\omega_2} \times \{2\} \big) \cup  \big(\underbrace{\{2,4\}}_{\omega_3} \times \{3\} \big) \nonumber \\ & \cup  \big(\underbrace{\{1,2,4,5,6\}}_{\omega_4} \times \{4\} \big) \cup  
\big(\underbrace{\{1,3,5\}}_{\omega_5} \times \{5\} \big) \nonumber
\end{align}  
\end{ex}

%%%%
\subsection{Supports of Linkage Matching Fields (SLMF's)} A combinatorial structure that lies in the heart of our families of patterns is the so-called \emph{Support of a Linkage Matching Field (SLFM)}. This was introduced in \cite{sturmfels1993maximal} for seemingly very different reasons than ours, relating to the theory of Gr\"obner bases. With $\# \mathcal{T}$ the size of a set $\mathcal{T}$, the definition reads:

\begin{definition}[\cite{sturmfels1993maximal}, paraphrased] \label{dfn:SLMF}
An $(r,m)$-SLMF is a set $$\Phi =\bigcup_{j \in [m-r]} \varphi_j \times \{j\} \subset [m] \times [m-r],$$ with the $\varphi_j$'s subsets of $[m]$ of size $r+1$, satisfying
\begin{align} 
\# \bigcup_{j \in \mathcal{T}} \varphi_j \ge \# \mathcal{T} + r, \, \, \,  \forall \mathcal{T} \subseteq [m-r]. \label{eq:SLMF}
\end{align}
\end{definition}

Let us provide some intuition on how $(r,m)$-SLMF's look like. Thinking about their defining condition \eqref{eq:SLMF}, one realizes that an $(r,m)$-SLMF can be thought of as an $m \times (m-r)$ matrix of $0$'s and $1$'s with not too large zero submatrices. In fact, $(r,m)$-SLMF's have an elegant algebraic characterization, which makes this intuition precise:

\begin{prp}[Proposition 4.1 in \cite{sturmfels1993maximal}, paraphrased] \label{prp:SLMF-algebraic}
Let $\Phi=\bigcup_{j \in [m-r]} \varphi_j \times \{j\}$ be a subset of $[m] \times [m-r]$ with $\#\varphi_j=r+1$ for every $j \in [m-r]$. Let $B=(b_{ij})$ be a generic $m \times (m-r)$ matrix supported on $\Phi$, that is  $b_{ij}=0$ for every $(i,j) \not\in \Phi$. Then $\Phi$ is an $(r,m)$-SLMF if and only if every $(m-r) \times (m-r)$ determinant of $B$ is non-zero. 
\end{prp}

Here is a concrete example:

\begin{ex} \label{ex:SLMF-Phi1-Phi2}
The following $\Phi_1, \Phi_2$ are $(2,6)$-SLMF's:
$$\Phi_1 = \begin{bmatrix} 
1 & 1 & 1 & 0 \\
1& 1 & 1 & 0 \\
1 & 0 & 0 & 0 \\
0 & 1 & 0 & 1 \\
0 & 0 & 1 & 1 \\
0& 0& 0& 1 
\end{bmatrix} $$
$$\cong  \underbrace{\{1,2,3\}}_{\varphi_1^1} \times \{1\} \cup \underbrace{\{1,2,4\}}_{\varphi_2^1} \times \{2\}  \cup \underbrace{\{1,2,5\}}_{\varphi_3^1} \times \{3\} \cup \underbrace{\{4,5,6\}}_{\varphi_4^1} \times \{4\} $$
$$\Phi_2 = \begin{bmatrix} 
0& 1 & 1 & 1  \\
1& 1 & 1 & 0 \\
0 & 0 & 0 & 1 \\
1 & 1 & 0 & 0 \\
0 & 0 & 1 & 1 \\
1& 0& 0& 0 
\end{bmatrix}$$
$$\cong  \underbrace{\{2,4,6\}}_{\varphi_1^2} \times \{1\} \cup \underbrace{\{1,2,4\}}_{\varphi_2^2} \times \{2\}  \cup \underbrace{\{1,2,5\}}_{\varphi_3^2} \times \{3\} \cup \underbrace{\{1,3,5\}}_{\varphi_4^2} \times \{4\} $$
\end{ex}

Inasmuch as there are ${m \choose m -r}$ determinants to compute in Proposition \ref{prp:SLMF-algebraic}, a practically relevant question is how can we efficiently check that a $\Phi \subset [m] \times [m-r]$ is an $(r,m)$-SLMF. Fortunately, this can be done by a simple linear algebra test:

\begin{prp} \label{prp:SLMF-V-check}
Let $\Phi =\bigcup_{j \in [m-r]} \varphi_j \times \{j\} \subset [m] \times [m-r]$ with $\#\varphi_j=r+1$ for every $j \in [m-r]$. Let $\V$ be a generic subspace of $\RR^m$ of dimension $r$. Construct an $m \times (m-r)$ matrix $B_\V$ supported on $\Phi$ as follows. For every $j \in [m-r]$ the projection of $\V$ onto the coordinates of $\RR^m$ indexed by $\varphi_j$ is a hyperplane with normal vector $h_j \in \RR^{r+1}$. Now let the $j$th column of $B_\V$ to contain the $i$th entry of $h_j$ at position $\varphi_{ij}$; here $\varphi_{ij}$ denotes the $i$th element of the ordered set $\varphi_j$. Then $\Phi$ is an $(r,m)$-SLMF if and only if $B_\V$ has full column-rank. 
\end{prp}

\noindent To get each normal vector $h_j$ in Proposition \ref{prp:SLMF-V-check}, we need a Singular Value Decomposition (SVD) of an $(r+1) \times r$ matrix (the projection of a basis of $\V$ onto the $r+1$ coordinates $\varphi_j$). In total, we must compute $m-r$ normal vectors; thus the complexity of forming the matrix $B_\V$ is $\mathcal{O}(r^3(m-r))$. To check that the $m \times (m-r)$ matrix $B_\V$ has full rank we need another SVD. Thus the total complexity of verifying that $\Phi$ is an $(r,m)$-SLMF is $\mathcal{O}\big((m-r)[r^3 + m(m-r)]\big)$.

We close with a useful observation:
\begin{remark} \label{remark:random-SLMF}
Theorem 2 in \cite{pimentel2015deterministic} suggests that $(r,m)$-SLMF's occur with high probability under uniform sampling in the low-rank regime $r \le m/6$. 
\end{remark}

%%%%
\subsection{The Patterns}

For every $j \in [n]$ we let $\Omega_j$ denote the set of all subsets of $\omega_j$ of size $r+1$. The following definition will be convenient in describing our observation patterns:

\begin{definition} \label{dfn:SLMF-induced}
Let $\Omega = \bigcup_{j \in [n]} \omega_j \times \{j\} \subset [m] \times [n]$ be an observation pattern. For a subset $\J \subset [n]$, we say that the sub-pattern $\Omega_\J = \bigcup_{j \in \J} \omega_j \times \{j\}$ induces the $(r,m)$-SLMF $\Phi = \bigcup_{k \in [m-r]} \varphi_k \times \{k \}$, if $\varphi_k \in \bigcup_{j \in \J} \Omega_j, \forall k \in [m-r]$.
\end{definition}

Here is our family of finitely completable patterns\footnote{We really mean \emph{generically finitely completable} patterns, but for the sake of brevity, we have dropped the attribute \emph{generically} here and in the rest of this section.}:

\begin{thm} \label{thm:finite}
Suppose that $\Omega \subset [m] \times [n]$ satisfies the following two conditions. First $\# \omega_j \ge r$ for every $j \in [n]$. Second, there exists a partition $[n] = \bigcup_{\nu \in [r]} \J_\nu$ of $[n]$ into $r$ subsets $\J_\nu$, such that for every $\nu \in [r]$ we have that $\Omega_{\J_\nu}$ induces an $(r,m)$-SLMF $\Phi_\nu$. Then a generic rank-$r$ matrix $X$ observed at $\Omega$ has finitely many rank-$r$ completions. 
\end{thm} 

In Theorem \ref{thm:finite}, if it so happens that each $\Omega_{\J_\nu}$ induces the same $(r,m)$-SLMF, i.e. all $\Phi_\nu$'s are equal, then the pattern is uniquely completable:

\begin{thm} \label{thm:same-SLMF}
Suppose that $\Omega \subset [m] \times [n]$ satisfies the conditions of Theorem \ref{thm:finite}. Suppose in addition that $\Phi_\nu = \Phi$ for every $\nu \in [r]$, where $\Phi$ is an $(r,m)$-SLMF. Then the only generic completion of a generic rank-$r$ matrix $X$ observed at $\Omega$ is $X$ itself. 
\end{thm} 

If the somewhat restrictive condition in Theorem \ref{thm:same-SLMF} requiring $r$ disjoint parts of the pattern to induce the same SLMF is removed, then unique completion can be achieved by allowing for a larger pattern, in the sense that now $r+1$ disjoint parts of the pattern each induce some SLMF:

\begin{thm} \label{thm:unique}
Suppose that $\Omega \subset [m] \times [n]$ satisfies the following two conditions. First $\# \omega_j \ge r$ for every $j \in [n]$. Second, there exists a partition $[n] = \bigcup_{\nu \in [r+1]} \J_\nu$ of $[n]$ into $r+1$ subsets $\J_\nu$, such that for every $\nu \in [r+1]$ we have that $\Omega_{\J_\nu}$ induces an $(r,m)$-SLMF. Then the only generic completion of a generic rank-$r$ matrix $X$ observed at $\Omega$ is $X$ itself. 
\end{thm} 

By \emph{generic} in the theorems above, we mean that the property in question (finite or unique completion) holds true for any rank-$r$ matrix $X$ that does not satisfy certain polynomial equations that depend only on $r,m,n,\Omega$. The formal way to say this is that there is a dense open set $\U \subset \M(r,m\times n)$ in the Zariski topology (\S \ref{subsection:Zariski}) such that for any $X \in \U$ the property holds true. Thus Theorem \ref{thm:finite} says that there is a dense open set $\U_1 \subset \M(r,m \times n)$ such that under the stated condition any $X \in \U_1$ observed at $\Omega$ has finitely many rank-$r$ completions. Theorem \ref{thm:unique} says that there is a dense open set $\U_2 \subset \M(r, m \times n)$ such that for every $X \in \U_2$, the only completion in $\U_2$ of $X$ observed at $\Omega$ is $X$ itself. Describing the equations that $\U_1,\U_2$ must avoid is too complicated. But $\U_1,\U_2$ have the important property of being dense, which in probabilistic terms implies that a matrix randomly sampled from $\M(r,m\times n)$ under a continuous non-degenerate distribution will lie in $\U_1 \cap \U_2$ with probability $1$. From a practical point of view, it is extremely unlikely that matrices occurring in applications, which may very well have special structures dictated by the application, satisfy the polynomial equations that would place them outside of $\U_1 \cap \U_2$. 

The following example illustrates Theorems \ref{thm:finite} and \ref{thm:unique}, some relationships and limitations:

\begin{ex} \label{ex:Thm}
The observation pattern $\Omega$ of Example \ref{example:main} satisfies the conditions of Theorem \ref{thm:finite} and thus it is finitely completable. To see this, consider the partition of $[5]$ into sets $\J_1 = \{1,2\}$ and $\J_2= \{3,4,5\}$. Now observe that the first three columns of $\Phi_1$ in Example \ref{ex:SLMF-Phi1-Phi2} are associated to $\Omega_1$, the set of all subsets of $\omega_1$ of size $r+1=3$. Moreover, $\Omega_2=\{\omega_2\}$ and this is the same as the fourth column of $\Phi_1$. Thus the first two columns of $\Omega$ induce the $(2,6)$-SLMF $\Phi_1$. Similarly, the first three columns of $\Phi_2$ lie in $\Omega_4$, while the fourth column of $\Phi_2$ is $\omega_5$. 

Note that $\omega_3$ was irrelevant in this consideration. This means that the submatrix of $\Omega$ obtained by removing $\omega_3$ is already finitely completable in its own right, while column 3 has just enough observed entries to be completed uniquely from the column-space of a completion of the submatrix (Lemma \ref{lem:vector-completion}). 

Next, the dimension of $\M(2, 6 \times 5)$ is $2(6+5-2)=18$ and $\Omega$ has exactly $18$ observed locations. This means that $\Omega$ is minimal, in the sense that if we remove any of the $1$'s, then it becomes infinitely completable at rank $2$. 

It is not difficult to check that $\Omega$ does not satisfy the conditions of Theorem \ref{thm:unique}. Instead, the following modified $6 \times 6$ pattern does:

$$ \Omega' = \begin{bmatrix} 
1 & 0 & \underline{1} & 1 & 1 & \underline{1}\\
1& 0 & 1 & 1 & 0 & \underline{1}\\
1& 0 & \underline{1} & 0 & 1 & 0\\
1 & 1 & 1 & 1 & 0 & 0  \\
1 & 1 & 0 & 1 & 1 & \underline{1}\\
0 & 1 & 0 & 1 & 0 & \underline{1}
\end{bmatrix}$$

\noindent The partition of $[6]$ given by $\J_1=\{1,2\}, \J_2=\{4,5\}, \J_3=\{3,6\}$ satisfies the conditions of Theorem \ref{thm:unique} with $\Phi_1, \Phi_2$ as above and $\Phi_3$ given by $\varphi_{1}^3=\{1,2,3\}, \, \varphi_{2}^3=\{1,2,4\}, \, \varphi_{3}^3=\{1,2,5\}, \, \varphi_{4}^3=\{1,2,6\}$. Thus $\Omega'$ is uniquely completable.

On the other hand, a computation with the algebraic geometry software \texttt{Macaulay2} reveals the fascinating fact that $\Omega$ is itself uniquely completable! 
\end{ex}

As Example \ref{ex:Thm} suggests, columns of $\Omega$ with exactly $r$ $1$'s are harmless in that they do not affect the property in question of the pattern if they are removed or if they remain present. On the other hand, fewer than $r$ observations per column immediately lead to infinitely many completions of rank $r$ and this is the point of condition $\# \omega_j \ge r$ in Theorems \ref{thm:finite}-\ref{thm:unique}. The condition involving the partition of the pattern into SLMF-inducing sub-patterns may appear more mysterious, but it has a clear geometric interpretation: it is the main idea behind the proof of Theorem \ref{thm:finite} (\S \ref{section:proofs}). 

\myparagraph{Verifying the conditions} We now discuss the verification of the conditions. This involves two ingredients.

The first ingredient is how to verify that a sub-pattern $\Omega_\J = \bigcup_{j \in \J} \omega_j \times \{j\}$ of $\Omega$, where $\J$ is a subset of $[n]$, induces an $(r,m)$-SLMF. This can be done in a simple and (relatively) efficient fashion. First, we remove any columns of $\Omega_\J$ that have less than $r+1$ ones. Thus we may assume that $\# \omega_j \ge r+1$ for every $j \in \J$. Secondly, we replace each $\omega_j$ with $\alpha_j=\#\omega_j - r$ subsets of $\omega_j$, each of size $r+1$. These $\alpha_j$ subsets can be constructed as follows. Select any subset $\I \subset \omega_j$ of size $r$. Then take the subsets to be $\I \cup \{ k\}$ for every $k \in \omega_j \setminus \I$. With this, we have a collection of $\alpha=\sum_{j \in \J} \alpha_j$ subsets $\varphi_1,\dots,\varphi_\alpha$ of $[m]$, each of size $r+1$. Now, we apply the method of Proposition \ref{prp:SLMF}: let $\V$ be a generic linear subspace of $\RR^m$ of dimension $r$. Its projection onto the $r+1$ coordinates of $\varphi_j$ is a hyperplane $h_j$ of $\RR^{r+1}$. Form an $m \times \alpha$ matrix $H$ by placing the $i$th entry of $h_j$ at the $(\varphi_{ij},j)$ entry of $H$, where $\varphi_{ij}$ is the $i$th entry of $\varphi_j$; the rest entries of $H$ are set to zero. Now, it is a straightforward consequence of the theory that led to Proposition \ref{prp:SLMF}, that $\Omega_\J$ induces an $(r,m)$-SLMF if and only if the matrix $H$ has rank $(m-r)$. Assuming that $\alpha \le m$, the complexity of this ingredient is $\mathcal{O}(\alpha r^3 + \alpha^2 m)$. 

The second ingredient consists of searching for a partition of the columns of $\Omega$ into sub-patterns, such that each induces an $(r,m)$-SLMF. Presently, we are not aware of any efficient and provably correct algorithm for performing this search. But the following empirical procedure appears to be effective, at least for small instances of the problem: randomly sample candidate partitions and test each cell of the partition using the method of the previous paragraph, until either a suitable partition is found or a specified time budget is exceeded. Here is some good news: We have experimentally observed that when one partition \emph{works}, then many partitions \emph{work}. Intuitively, this is because the SLMF property fails when there are large zero submatrices present in the sub-pattern. But when $\Omega$ satisfies the condition of the theorem, randomly partitioning the columns of $\Omega$ tends to make it unlikely for such large zero submatrices to appear in the sub-patterns of the partition.   

%%%
\subsection{Designing Uniquely Completable Patterns}
Theorems \ref{thm:same-SLMF}-\ref{thm:unique} offer a device for constructing economic (close to minimal) uniquely completable patterns, providing one has a means of producing sub-patterns that induce SLMF's. This is significant for inductive matrix completion, where the observation pattern is part of the machine learning system (\S \ref{subsection:LRMC}, \S\ref{subsection:Deterministic Patterns}). The examples of this section progressively develop ideas in this direction. As a general remark, we note that if a pattern 1) contains a uniquely completable pattern as a sub-pattern, and 2) contains at least $r$ observations  per column, then it is itself uniquely completable (the reason is as in the second paragraph of Example \ref{ex:Thm}).

In Examples \ref{ex:r copies of SLMF}-\ref{ex:r+1 SLMF's} the sub-patterns that induce SLMF's are SLMF's themselves:

\begin{ex} \label{ex:r copies of SLMF}
Let $\Phi$ be any $(r,m)$-SLMF. Viewing $\Phi$ as an $m \times (m-r)$ matrix, it follows from Theorem \ref{thm:same-SLMF} that any $\Omega$ obtained by concatenating $r$ copies of $\Phi$, i.e. $\Omega = [ \Phi_1 \cdots \Phi_r] \Pi,$ with $\Phi_i = \Phi$ and $\Pi$ any permutation, is a uniquely completable pattern. Note also that this is a minimal pattern, since it has $(r+1)(m-r)r = \dim \M(r, m \times r(m-r))$ observed entries. 
\end{ex}

\begin{ex} \label{ex:r+1 SLMF's}
Let $\Phi_1,\dots,\Phi_{r+1}$ be $(r,m)$-SLMF's. It follows from Theorem \ref{thm:unique} that $\Omega = [ \Phi_1 \cdots \Phi_{r+1}] \Pi$, with $\Pi$ any permutation, is a uniquely completable pattern. This is never a minimal pattern, as its number of observations is $(r+1)^2(m-r) > \dim \M(r,m \times (m-r)(r+1)) = r(m-r)(r+2)$; however it is close to minimal. 
\end{ex}

In the next example the sub-patterns that induce SLMF's are fully observed columns.

\begin{ex} \label{ex:cross pattern}
Consider an observation pattern with the following structure. Let $\I$ and $\J$ be any subsets of $[m]$ and $[n]$, respectively, of size $r$. Then let $\Omega = (\I \times [n]) \cup ([m] \times \J) \subset [m] \times [n]$. 

An instance of this pattern for $m=6, n=5, r=2$ is with $\I=\{3,4\}$ and $\J=\{2,3\}$:
$$\Omega = \begin{bmatrix} 
0 & 1 & 1 & 0 & 0 \\
0& 1 & 1 & 0 & 0 \\
1& 1 & 1& 1 & 1 \\
1 & 1 & 1 & 1 & 1 \\
0 & 1 & 1 & 0 & 0\\
0 & 1 & 1 & 0 & 0
\end{bmatrix}$$

Such a pattern is minimal as it has $r(m+n-r) = \dim \M(r,m \times n)$ entries. Moreover, it is easy to see \textemdash{and it is well known\textemdash} that such a pattern is uniquely completable. Indeed, the pattern allows for the full observation of a basis of the column-space of $X$, and since the remaining columns contain sufficiently many observations, they can each be completed by solving a linear system of equations (e.g., see Lemma \ref{lem:vector-completion}).

Such an $\Omega$ is captured by Theorem \ref{thm:same-SLMF}. Indeed, there are $r$ fully observed columns, each of which can be seen to induce any $(r,m)$-SLMF $\Phi$: for $j \in \J$ we have $\omega_j = [m]$; thus $\Omega_j$ is the set of all subsets of $[m]$ of size $r+1$. Hence, if $\Phi =\bigcup_{k \in [m-r]} \varphi_k \times \{k\} \subset [m] \times [m-r]$ is any $(r,m)$-SLMF, then $\varphi_k \in \Omega_j$ for any $k \in [m-r]$. This is true for every $j \in \J$, and since $\# \J=r$, the conditions of Theorem \ref{thm:same-SLMF} are satisfied.
\end{ex}

Combining Examples \ref{ex:r+1 SLMF's}-\ref{ex:cross pattern}, we can concatenate SLMF's with fully observed columns:

\begin{ex} \label{ex:SLMF+full columns}
Let $\Phi$ be any $(r,m)$-SLMF viewed as a matrix. Let $\mathfrak{e}$ be the $m$-column vector of all $1$'s. It follows from Theorem \ref{thm:unique} that the concatenation of $k$ copies of $\Phi$ with $r+1-k$ copies of $\mathfrak{e}$ is a uniquely completable pattern. For $m=6$ and $r=2$, such an instance is $\Omega = [\Phi_1 \, \Phi_1 \, \mathfrak{e}]$, where $\Phi_1$ is the SLMF that appears in Example \ref{ex:SLMF-Phi1-Phi2}. More generally, one can concatenate a total of $k$ replicas of any SLMF's together with $r+1-k$ copies of $\mathfrak{e}$; e.g. $\Omega = [\Phi_1 \, \Phi_2 \, \mathfrak{e}]$, where $\Phi_1$ and $\Phi_2$ are as in Example \ref{ex:SLMF-Phi1-Phi2}. 
\end{ex}

A limitation in designing uniquely completable patterns using the method of Example \ref{ex:SLMF+full columns} is that the matrix size of the pattern is constrained by the size of its building blocks (it has $k(m-r)+r+1-k$ columns, where $0 \le k \le r+1$ is the number of SLMF's that were employed). This comes from the fact that the SLMF-inducing sub-patterns that we used in Example \ref{ex:SLMF+full columns} are extreme: a building block is either an SLMF or it is a fully observed column. Another limitation is that the method is too restrictive, as it asks for $r+1-k$ columns to be fully observed. The next example demonstrates a method that easily addresses this issue. 

\begin{ex} \label{ex:6 x 10}
Suppose we are given the $(2,6)$-SLMF $\Phi_2$ of Example \ref{ex:SLMF-Phi1-Phi2} and we are asked to produce a uniquely completable pattern $\Omega$ for rank-$2$ matrices of size $6 \times 10$, which has $32$ entries but it can not afford to fully observe any column. The idea is to use Theorem \ref{thm:unique}, which asks for the partition of the $10$ columns of $\Omega$ into $r+1=3$ subsets, each inducing $\Phi_2$. But since no column can be fully observed, we can not use the vector $\mathfrak{e}$ of all ones as a column of $\Omega$. On the other hand, concatenating three copies of $\Phi_2$ gives a pattern with $12$ columns as opposed to the required $10$ columns. The crucial observation here is that, starting from $\Phi_2$, we can construct a sub-pattern with number of columns ranging from $1$ to $4=m-r$ that always induces $\Phi_2$. The technique is to glue two columns of $\Phi_2$ along $r=2$ common entries to yield a single column. We demonstrate this by producing a sub-pattern with two columns that induces $\Phi_2$, which we will then concatenate with two copies of $\Phi_2$. Since $\varphi_1^2 = \{2,4,6\}$ and $\varphi_2^2 = \{1,2,4\}$ have the $r=2$ common elements $2$ and $4$, we glue them to obtain a new subset $\varphi_{12}^2 = \{1,2,4,6\}$. This gives the sub-pattern 
$$\Phi_2' = \begin{bmatrix} 
 1 & 1 & 1  \\
 1 & 1 & 0 \\
 0 & 0 & 1 \\
 1 & 0 & 0 \\
 0 & 1 & 1 \\
 1 & 0 & 0
\end{bmatrix}=[\varphi_{12}^2 \, \, \varphi_3^2 \, \, \varphi_4^2],$$ which is easily seen to induce $\Phi_2$. Still this has three columns, so we need another gluing. Since $\varphi_3^2 = \{1,2,5\}$ and $\varphi_4^2 = \{1,3,5\}$ have the $r=2$ common elements $1$ and $5$, we glue them to obtain a new subset $\varphi_{34}^2 = \{1,2,3,5\}$. This gives the sub-pattern
$$\Phi_2'' = \begin{bmatrix} 
 1 & 1   \\
 1 & 1  \\
 0 & 1 \\
 1 & 0  \\
 0 & 1  \\
 1 & 0
\end{bmatrix}=[\varphi_{12}^2 \, \, \varphi_{34}^2],$$ that induces $\Phi_2$. Finally, the required rank-$2$ uniquely completable $6 \times 10$ pattern with $32$ entries is $[\Phi_2 \, \, \Phi_2 \, \, \Phi_2'']$. This is not unique; an alternative choice is $[\Phi_2' \, \, \Phi_2' \, \, \Phi_2]$.
\end{ex}

It is clear from Examples \ref{ex:r copies of SLMF}-\ref{ex:6 x 10}, that the ability to design uniquely completable patterns is contingent \textemdash{as per the present theory\textemdash} on the ability to design SLMF's. A method to do this is via  random sampling (Remark \ref{remark:random-SLMF}) coupled with the criterion of Proposition \ref{prp:SLMF}. That is, we randomly set to one $r+1$ entries for each column of the zero $m \times (m-r)$ matrix and check by Proposition \ref{prp:SLMF} if the resulting $\Phi$ is an $(r,m)$-SLMF; we repeat until an SLMF is obtained. 

Alternatively, it is easy to specify SLMF's for any $r,m$ and use them as prototypes for synthesizing $\Omega$'s. For instance, let $\Phi_{\P,r,m} = [E ; I]$ be the $m \times (m-r)$ matrix whose top $r \times (m-r)$ block $E$ consists of $1$'s and its bottom $(m-r) \times (m-r)$ block is the identity matrix. The subscript $\P$ in $\Phi_{\P,r,m} $ stands for \emph{pointed}, the name assigned to this special SLMF in \cite{sturmfels1993maximal}. Definition \ref{dfn:SLMF} can be immediately used to verify that $\Phi_{\P,r,m}$ is an $(r,m)$-SLMF. In fact, $\Pi \Phi_{\P,r,m} $ is also an $(r,m)$-SLMF for any row permutation $\Pi$. Moreover, any two columns of $\Pi \Phi_{\P,r,m} $ overlap at $r$ row indices and thus can be conveniently glued as in Example \ref{ex:6 x 10}, to produce sub-patterns $(\Pi \Phi_{\P,r,m})'$ with number of columns ranging from $1$ to $m-r$, that still induce $\Pi \Phi_{\P,r,m} $. Using different permutations, we can obtain sub-patterns $(\Pi_1 \Phi_{\P,r,m} )',\dots,(\Pi_{r+1} \Phi_{\P,r,m} )'$, which 1) can have varying column sizes, 2) can have different observation densities across different regions, and 3) they each induce an SLMF. It is a consequence of Theorem \ref{thm:unique} that $\Omega=[(\Pi_1 \Phi_{\P,r,m} )' \cdots (\Pi_{r+1} \Phi_{\P,r,m} )']$ is uniquely completable. The observation patterns that we just described already form a rich family that can be used for any $r,m,n$ in practice. The attribute \emph{rich} refers to the fact that we can construct many different patterns just by the operations of gluing columns and permuting rows of the $(r,m)$-SLMF $\Phi_{\P,r,m}$. The next example illustrates this:

\begin{ex} \label{ex:alternatively-Omega}
Consider the $6 \times 6$ rank-$2$ pattern $\Omega'$ of Example \ref{example:main}. This was shown to be uniquely completable due to Theorem \ref{thm:unique}, by noting that $\Omega'_{\{1,2\}}$ induces $\Phi_1$, $\Omega'_{\{4,5\}}$ induces $\Phi_2$, and $\Omega'_{\{3,6\}}$ induces $\Phi_{\P,2,4}$ ($\Phi_1, \Phi_2$ are as in Example \ref{ex:SLMF-Phi1-Phi2}). In fact, $\Omega'$ can be shown to be uniquely completable using only the SLMF 
$$\Phi_{\P,2,6} = \begin{bmatrix} 
 1 &  1  & 1 &    1   \\
 1 &  1  & 1 &    1 \\
 1 & 0  & 0 & 0  \\
 0 & 1  &  0 & 0 \\
 0 & 0 &  1  & 0 \\
 0 & 0 &  0 & 1
\end{bmatrix}.$$ 
Indeed, gluing the first three columns of $\Phi_{\P,2,6}$ along row indices $1,2$, and permuting rows $1,2$ with rows $4,5$, gives $\Omega'_{\{1,2\}}$. Moreover, gluing columns $1,3,4$ of $\Phi_{\P,2,6}$ along row indices $1,2$, and transposing rows $2,5$ and rows $3,4$ gives $\Omega'_{\{4,5\}}$. Finally,  gluing columns $1,2$ and columns $3,4$ of $\Phi_{\P,2,6}$ gives $\Omega'_{\{3,6\}}$.
\end{ex}

%%%
\subsection{Regimes}

We briefly place the deterministic results of this paper into a wider context. Our purpose is to qualitatively \textemdash{and certainly non-exhaustively\textemdash} discuss the complementary nature of the theory with respect to existing ones.

\myparagraph{Rank} Most works in the LRMC literature that study random observation patterns, either explicitly or implicitly bound the rank $r$. For instance, \cite{keshavan2010matrix} requires the explicit bound $r \le \log (\max\{m,n\})$. On the other hand, \cite{recht2011simpler} asks $\# \Omega \ge 32 \cdot \max\{\mu_1^2,\mu_0\} r (m +n) \beta \log^2(2 \max\{m,n\})$, where $\beta >1$ and $\mu_0,\mu_1$ are incoherence parameters with $1 \le \mu_0 \le \max\{m,n\}/r$. Since $mn \ge \# \Omega $ and $\max\{\mu_1^2,\mu_0\} \beta \ge 1$, this enforces the implicit bound on the rank $r \le mn  / 32 (m+n) \log^2(2 \max\{m,n\})$. In other words, to provably complete a matrix via nuclear norm minimization, the matrix needs to be low-rank; hence the attribute \emph{low-rank} matrix completion. This is not a surprise, as the nuclear norm is the convex envelope of the rank. But for $m=6, \, n=5$ in our running Example \ref{example:main}, both bounds above ask the rank to be less than $1$, i.e. a $6 \times 5$ rank-$1$ matrix is not sufficiently low rank !  

In contrast, Theorems \ref{thm:finite}-\ref{thm:unique} hold for \emph{any rank}. For instance, if one unfolds the condition of Theorem \ref{thm:finite} for the extreme case of maximal bounded rank, say $n \ge m$ and $r = m-1$, one obtains a meaningful answer \textemdash the well-known fact that finite completion in this case requires full observation of $r$ columns of $X$. For other high values of the rank, our theorems give novel families of finitely and uniquely completable patterns, a regime, not only out of bounds for the probabilistic theories, but one where the nuclear norm minimization itself is expected to fail. Put differently, our theoretical results cover the high-rank regime as well, for which efficient matrix completion algorithms seem to not be available yet. This is not to be confused with what is known as \emph{high-rank matrix completion}  \cite{eriksson2012high,elhamifar2016high,ongie2017algebraic,ongie2021tensor}: That part of the literature refers to completing a matrix whose columns come from a union of low-dimensional linear subspaces. There, even though the rank of the complete matrix can be high, the intrinsic dimension of the geometric model of the data is still low (the dimension of a union of linear subspaces is equal to the maximum dimension among the dimensions of the subspaces), and this is exactly what existing theories and algorithms rely upon. Ultimately, as data representations become more compact, say via deep learning architectures such as autoencoders, the (relative) rank of the data becomes higher, and high-rank subspace learning methods, including matrix completion, will become important. 

\myparagraph{Number of observations} The state-of-the-art in the theory of LRMC that studies random patterns, considers $\mathcal{O}(\dim \M(r, m \times n))$ observed entries (up to logarithmic factors). This optimal sampling rate becomes effective in the low-rank regime. For instance, for $m=6, \, n = 5$ and $r=2$, the above bound of \cite{recht2011simpler} asks that one samples at least $32 r (m+n) = 32 \cdot 2 \cdot (6+5) = 704$ entries ! In contrast, Theorems \ref{thm:finite}-\ref{thm:same-SLMF} encompass patterns with the absolutely minimal number of observations, equal to $\dim \M(r, m \times n)=r(m+n-r)$ for any rank $r$; such is the finitely completable pattern of Example \ref{example:main} and the family of uniquely completable patterns of Example \ref{ex:r copies of SLMF}. 

At the same time, the fact that we consider unique completion in the generic sense \textemdash{that is, we ignore a zero-measure set of pathological matrices\textemdash} allows us to have even more economic patterns as without this convention. To illustrate this, the main result of \cite{xu2018minimal} implies that there exist $m \times m$ matrices of rank $\le r$, with $r \le m/2$, that are not uniquely completable from any observation pattern $\Omega$ with $\# \Omega <  4mr - 4 r^2$. For $r=2$ and $m=6$, this asks for at least $4\cdot6\cdot 2 - 4 \cdot 2 \cdot 2 = 32$ observed entries. On the other hand, the $\Omega'$ of Example \ref{ex:Thm} is a $6 \times 6$ rank-$2$ pattern with $24$ observed entries, which is \emph{generically} uniquely completable.

%%%%%%
\section{Algebraic Geometry Preliminaries} \label{section:AG}

In this section we develop certain algebraic geometry notions \cite{cox2013ideals,harris2013algebraic} that we need. Along the way, we build intuition and state useful facts about the problem of LRMC. 
Our entire discussion is over $\RR$, in contrast to classical algebraic geometry which is over $\CC$. Since $\CC$ is algebraically closed, many theorems in algebraic geometry that hold over $\CC$ do not hold over $\RR$ and care is needed when working with the latter. But in many cases one can get the same statements over $\RR$ if one works with \emph{schemes} \cite{hartshorne2013algebraic,vakil2017rising} instead of varieties and then restricts attention to the so-called \emph{$\RR$-valued points of the scheme}; this is what we have implicitly done whenever needed in this paper. Another possibility is to first establish the desired result over $\CC$ and then transfer it to $\RR$; this approach was followed in the earlier version of this manuscript \cite{tsakiris2020exposition}, and also in the work \cite{breiding2021algebraic} of the related problem of algebraic compressed sensing. Finally, LRMC naturally connects with algebraic geometry and combinatorial notions that are beyond the scope of the present paper; such are \emph{Gr\"obner bases} and \emph{matroids}. For an approach that incorporates these, we refer to \cite{matroid-Tsakiris-23}.

%%%%%%
\subsection{Algebraic varieties and the Zariski topology} \label{subsection:Zariski}

An algebraic variety $\Y$ of $\RR^{k}$ is the set of solutions of a system of polynomial equations in $k$ variables. More precisely, $\Y$ is called an affine algebraic variety, to distinguish from the notion of a projective algebraic variety, which we will also need. The projective space $\PP^{k-1}$ is the set of all equivalence classes of non-zero vectors in $\RR^k$, where two vectors are declared equivalent if they are equal up to scale. A projective variety $\Y$ of $\PP^{k-1}$ is the set of solutions in $\PP^{k-1}$ of a system of homogeneous polynomials; a polynomial is homogeneous if all monomials that appear in it have the same degree. Testing whether an equivalence class is a root of a homogeneous polynomial is defined by testing whether any representative of that class is a root; scale does not matter because the polynomial is homogeneous. 

To work with algebraic varieties (affine or projective), it is very convenient to use the so-called Zariski topology. The Zariski topology consists of a system of open and closed sets that satisfy the usual axioms of the more common Euclidean topology, that is 1) a set is open if and only if its complement is closed, 2) the ambient space and the empty set are both open and closed, 3) any union of open sets is open, and 4) the intersection of finitely many open sets is open. 

In the Zariski topology a set is defined to be closed if it is an algebraic variety. A very important property is irreducibility. An algebraic variety is called irreducible if it can not be written as the union of two proper subvarieties of it. Here is a consequence of irreducibility that we will use very often in the sequel:

\begin{prp} \label{prp:open}
Let $\Y$ be an irreducible algebraic variety (affine or projective). Let $\U$ be a non-empty open subset of $\Y$; then $\U$ is dense in $\Y$. Moreover, the intersection of finitely many non-empty open subsets of $\Y$ is non-empty and open and thus dense in $\Y$. 
\end{prp}

\noindent In Proposition \ref{prp:open} \emph{dense} is meant in the usual topological sense: a set $\U \subset \Y$ is dense if for every point  $y \in \Y$ and every open set that contains $y$ the open set intersects $\U$. In the language of probability, if $\U \subset \Y$ is a fixed open dense set, then a point randomly drawn from a non-degenerate continuous distribution on $\Y$ lies in $\U$ with probability $1$.

%%%%%%%
\subsection{The variety $\M(r, m \times n)$ of bounded-rank matrices} \label{subsection:determinantal}

We denote by $\M(r, m \times n) \subset \RR^{m \times n}$ the set of all $m \times n$ real matrices of rank at most $r$. By basic linear algebra, a matrix $X \in \RR^{m \times n}$ lies in $\M(r, m \times n)$ if and only if all $(r+1) \times (r+1)$ determinants of $X$ are zero (i.e., they \emph{vanish}). These determinants are polynomials of degree ${r+1}$ in the entries of $X$. Thus $\M(r, m \times n)$ is an algebraic variety of $\RR^{m \times n}$. In the Zariski topology the closed sets of $\M(r, m \times n)$ are defined to be the subvarieties of $\M(r, m \times n)$, that is $m \times n$ matrices of rank $\le r$ which satisfy additional polynomial equations besides the $(r+1) \times (r+1)$ determinants being zero. As usual, the open sets are defined to be complements of closed sets. Thus the open sets of $\M(r, m \times n)$ are defined by the non-simultaneous vanishing of polynomial equations.

\begin{ex}
The set $\M(r-1,m \times n)$ of $m \times n$ matrices of rank at most $r-1$ is a closed set of $\M(r, m \times n)$, defined by the vanishing of all $r \times r$ determinants. The complement of $\M(r-1, m \times n)$ in $\M(r, m \times n)$ is the set of all $m \times n$ matrices of rank exactly $r$ and it is by definition an open set of $\M(r, m \times n)$, because a matrix $X \in \M(r, m \times n)$ lies in that open set if and only if there is some $r \times r$ determinant of $X$ which is non-zero. 
\end{ex}

Given $\Omega \subseteq [m]\times [n]$, we identify $\RR^{ \Omega}$ with the set of $m \times n$ matrices which are supported on $\Omega$, i.e. they are zero in the complement of $\Omega$. We let $\pi_{\Omega}:  \M(r, m \times n) \rightarrow \RR^{\Omega}$ be the coordinate projection that preserves the entries of $X$ indexed by $\Omega$ and sets the rest to zero. The set of matrices $X' \in \M(r, m \times n)$ with $\pi_{\Omega}(X') = \pi_{\Omega}(X)$ is called the \emph{fiber of $\pi_{\Omega}$ over $\pi_{\Omega}(X)$}, denoted $\pi_{\Omega}^{-1}\big(\pi_{\Omega}(X)\big)$. Finding an at most rank-$r$ completion of $\pi_{\Omega}(X)$ is the same as finding an element $X' \in \pi_{\Omega}^{-1}\big(\pi_{\Omega}(X)\big)$. Finitely many at most rank-$r$ completions is the same as the fiber $\pi_{\Omega}^{-1}\big(\pi_{\Omega}(X)\big)$ being a finite set.

It turns out that $\M(r, m \times n)$ is irreducible and so Proposition \ref{prp:open} applies. Thus, when a property is true for every $X$ in a non-empty open set $\U$ of $\M(r, m \times n)$, we will say that the property is true for a generic $X$. In \S \ref{section:Introduction} and \S \ref{section:main} we were referring to an observation pattern as being \emph{finitely} or \emph{uniquely completable}. A more precise terminology is:

\begin{definition} \label{dfn:gfc}
$\Omega$ is generically finitely completable if there is an non-empty open set $\U \subset \M(r, m \times n)$ with $\pi_{\Omega}^{-1}\big(\pi_{\Omega}(X)\big)$ a finite set for every $X \in \U$. $\Omega$ is generically uniquely completable if there is a non-empty open set $\U \subset \M(r, m \times n)$ with $\pi_{\Omega}^{-1}\big(\pi_{\Omega}(X)\big) \cap U = \{X\}$ for every $X \in \U$. 
\end{definition}

A very important theorem in algebraic geometry with many applications is known as \emph{the upper-semicontinuity of the fiber dimension} \cite{hartshorne2013algebraic,vakil2017rising,harris2013algebraic}; the reader may think of this as an algebraic geometric analogue of the \emph{rank plus nullity theorem} in linear algebra. We will state some consequences of this theorem that we will need for our map of interest $\pi_\Omega: \M(r,m \times n) \rightarrow \RR^\Omega$. First, note that the dimension of $\M(r,m \times n)$ as an algebraic variety is $r(m+n-r)$. We have:

\begin{prp} \label{prp:upper}
$\Omega$ is generically finitely completable if and only if the dimension of the image of the map $\pi_\Omega: \M(r,m \times n) \rightarrow \RR^\Omega$ is $r(m+n-r) = \dim \M(r, m \times n)$. 
\end{prp}

By Proposition \ref{prp:upper} we see that $\#\Omega \ge r(m+n-r)$ is a necessary condition for finite or unique completability. In fact, it is a consequence of \emph{algebraic matroid theory} \cite{rosen2020algebraic,kiraly2015algebraic} that to understand all generically finitely completable patterns, it is enough to understand those that have minimal size equal to $r(m+n-r)$:

\begin{prp} \label{prp:matroid}
An $\Omega$ is generically finitely completable if and only if it contains an $\Omega'$ with $\#\Omega' = \dim \M(r,m \times n) = r(m+n-r)$ such that $\Omega'$ is generically finitely completable.
\end{prp} 

We close with another consequence of the \emph{fiber dimension theorem} that will come in handy in the proof of Theorem \ref{thm:finite}:

\begin{prp} \label{prp:X}
Let $\U$ be a non-empty open set of $\M(r,m \times n)$. If there exists an $X \in \U$ such that $\pi_{\Omega}(X)$ has finitely many completions in $\U$, that is the set $\pi_{\Omega}^{-1}\big(\pi_{\Omega}(X)\big) \cap \U$ is finite, then $\Omega$ is generically finitely completable.
\end{prp}

%%%
\subsection{Grassmannian variety $\Gr(r,m)$ and coordinates} \label{subsection:Gr}

\subsubsection{Pl{\"u}cker coordinates} The Grassmannian $\Gr(r,m)$ is defined as the set of all $r$-dimensional linear subspaces of $\RR^m$. Let $\V \in \Gr(r,m)$ and let $B \in \RR^{m \times r}$ contain a basis for $\V$ in its columns. Then $\V$ is certainly defined by $B$ as the column space $\C(B)$ of $B$, but the entries of the matrix $B$ do not serve as coordinates for $\V$, i.e. they are not discriminative features of $\V$, since any matrix $B S$ for invertible $S \in \RR^{r \times r}$ also satisfies $\V = \C(B S)$. Suppose for a moment that $\V$ is $1$-dimensional. Then $B=b \in \RR^m$ and a non-zero vector $b' \in \RR^m$ also defines $\V$ if and only if $b$ and $b'$ are equal up to scale. That is, $\V$ is  identified with a unique point of the projective space $\PP^{m-1}$ and that point serves as a set of (projective) coordinates for $\V$. 

The Pl{\"u}cker embedding $\Gr(r,m) \rightarrow \PP^N$ is a classical construction in algebraic geometry that generalizes this representation of $\V$ for any $r \ge 1$, by identifying $\V$ with a unique point of a higher dimensional projective space $\PP^{N}$ with $N={m \choose r}-1$. The coordinates of that point, called the Pl{\"u}cker coordinates of $\V$, are all the $r \times r$ determinants of any basis $B$ of $\V$. If $B_1, B_2 \in \RR^{m \times r}$ are two bases of $\V$, then there exists an $r\times r$ invertible matrix $S$ such that $B_2 = B_1 S$. Now the $r\times r$ determinants of $B_2$ are obtained by the $r\times r$ determinants of $B_1$ by multiplying them by $\det(S)$. Hence both bases will give the same point in $\PP^N$ under the Pl\"ucker embedding. A machine learning flavored view of the Pl\"ucker embedding is as a polynomial mapping that takes a linear subspace of dimension $r$ to a line in the higher dimensional ambient space $\RR^{N+1}$. 

Since specifying an $r \times r$ determinant of $B$ is the same as specifying $r$ rows of $B$ indexed by some $\psi=\{i_1<\dots <i_r\} \subset [m]$, we denote the Pl{\"u}cker coordinates of $\V$ by $[\psi]_{\V}$, where $\psi$ ranges among all ${m \choose r}$ subsets of $[m]$ of size $r$. 

The reader may wonder how one can distinguish points of $\PP^N$ that correspond to $r$-dimensional subspaces of $\RR^m$ from the rest of the points in $\PP^N$. These are precisely the roots of certain quadratic homogeneous polynomials, known as the \emph{Pl{\"u}cker relations}. That is, a point of $\PP^N$ is the image of an $r$-dimensional subspace $\V$ under the Pl{\"u}cker embedding if and only if it is a root of the Pl{\"u}cker relations. This identification of $\Gr(r,m)$ with points of $\PP^N$ satisfying these homogeneous polynomial equations makes $\Gr(r,m)$ a projective variety of dimension $r(m-r)$. It is also true that $\Gr(r,m)$ is irreducible and thus Proposition  \ref{prp:open} also applies.

We illustrate these notions with a series of examples.

\begin{ex}
$\Gr(1,3)$ is identified via the Pl{\"u}cker embedding with $\PP^2$, since there are no algebraic (Pl{\"u}cker) relations between the $1 \times 1$ determinants of a $3 \times 1$ matrix $B$. Indeed, by definition the points of $\PP^2$ are in 1-1 correspondence with the lines through the origin in $\RR^3$. Also $\dim\Gr(1,3)=2=1\cdot(3-1)$.
\end{ex}

\begin{ex} \label{ex:Gr(2,4)}
Consider the following $B \in \RR^{4 \times 2}$: 

$$B = \begin{bmatrix} 1 & 0 \\ 0 & 1 \\ 0 & 2 \\ 3 & 4 \end{bmatrix}$$

\noindent The column space $\C(B)$ of $B$ is a $2$-dimensional subspace $\V$ of $\RR^4$ uniquely defined by the following 
Pl{\"u}cker coordinates viewed as a point in the projective space $\PP^5$:  
\begin{align}
&\big[ [12]_{\V} : [13]_{\V} : [14]_{\V} : [23]_{\V} : [24]_{\V} :[34]_{\V}\big]= \nonumber \\ 
&\big[1:2:4:0:-3:-6\big] \in \PP^5 \nonumber
\end{align}

The points $\alpha=[\alpha_{12}:\alpha_{13}:\alpha_{14}:\alpha_{23}:\alpha_{24}:\alpha_{34}] \in \PP^5$ that correspond to $2$-dimensional subspaces of $\RR^4$ are those that satisfy the (unique for this case) Pl{\"u}cker relation
$$\alpha_{14} \alpha_{23}+\alpha_{12}\alpha_{34}-\alpha_{13}\alpha_{24} = 0$$
Indeed, for $\V=\mathcal{C}(B)$ above we have $$4 \cdot 0+1\cdot(-6)-2\cdot(-3)=0$$ Hence $\Gr(2,4)$ is a projective hypersurface of $\PP^5$ of dimension $4 = 2 \cdot (4-2)$. 

To conveniently talk about polynomials on $\Gr(2,4)$ we consider the polynomial ring $\RR\big[[12],[13],[14],[23],[24],[34]\big]$ with coefficients in $\RR$ and polynomial variables the $[ij]$'s with $i<j \in \{1,2,3,4\}$. With that convention $\Gr(2,4)$ is the hypersurface of $\PP^5$ defined by the Pl\"ucker polynomial (relation) $$[14][23]+[12][34]-[13][24]$$
\end{ex}

\begin{ex} \label{ex:Gr(2,6)}
$\Gr(2,6)$, the set of all $2$-dimensional subspaces of $\RR^6$, is identified via the Pl{\"u}cker embedding as an $8$-dimensional projective variety of $\PP^{14}$ defined by $15$ Pl{\"u}cker relations. 
\end{ex}

For $\omega \subset [m]$ we let $\pi_{\omega}: \RR^m \rightarrow \RR^{\omega}$ be the coordinate projection that takes a vector $v = (v_i)_{i \in [m]}$ to the vector $\pi_{\omega}(v) = (v_i)_{i \in \omega}$. As a first application of the Pl{\"u}cker coordinates, a simple inspection determines the coordinate projections under which $\V$ preserves its dimension: 

\begin{prp} \label{prp:Gr-non-drop}
Let $\V \in \Gr(r,m)$ and $\psi \subseteq [m]$ with $\#\psi = r$. Then $\dim \pi_{\psi}(\V) = r$ if and only if $[\psi]_{\V} \neq 0$. 
\end{prp}
\begin{proof}
Let $B=[b_1 \cdots b_r] \in \RR^{m \times r}$ be a basis of $\V$. Then $\pi_{\psi}(\V)$ is spanned by $\pi_{\psi}(b_j), \, j \in [r]$. Also $\dim \pi_{\psi}(\V) = r$ if and only if the $\pi_{\psi}(b_j)$'s are linearly independent. This is equivalent to the $r \times r$ row submatrix of $B$ indexed by $\psi$ being invertible, which in turn is by definition equivalent to the Pl{\"u}cker coordinate $[\psi]_{\V}$ being non-zero. 
\end{proof}

\begin{ex} 
Continuing with Example \ref{ex:Gr(2,4)}, all Pl{\"u}cker coordinates of $\V$ are non-zero except $[23]_{\V}$: the projection of $\V$ onto coordinates $\{2,3\}$ has dimension $1$. The projection of $\V$ onto any other two coordinates has dimension $2$. 
\end{ex} 

%%%%%%%%%%%
\subsubsection{Standard local coordinates on $\Gr(r,m)$} \label{subsection:local-coordinates-Gr}

If a subspace $\V \in \Gr(r,m)$ does not drop dimension upon projection onto the first $r$ coordinates, that is if $\V$ lies in the open set of $\Gr(r,m)$ on which the first Pl\"ucker coordinate $[\{1,\cdots,r\}]_{\V}$ is non-zero, then a canonical choice for a basis of $\V$ is possible. This is the unique basis $B_{\V}$ of $\V$ with the identity matrix on the top $r \times r$ block of $B$. The $(m-r) \times r$ bottom block of $B_{\V}$ contains the \emph{standard local coordinates} of $\V$, where \emph{local} signifies that $B_{\V}$ is uniquely determined by $\V$, providing that $[\{1,\cdots,r\}]_{\V} \neq 0$.  

\begin{ex}
Standard local coordinates on $\Gr(2,6)$ have the form 
$$B_{\V}=\begin{bmatrix} 
1 & 0 \\
0 & 1 \\
* & * \\
* & * \\
* & * \\
* & * 
\end{bmatrix} $$
\end{ex} 

\begin{ex}
In Example \ref{ex:Gr(2,4)} the bottom $2 \times 2$ part of the matrix $B$ consists of the standard local coordinates of $\V = \C(B) \in \Gr(2,4)$.
\end{ex}

%%%%%%
\subsubsection{Standard dual local coordinates on $\Gr(r,m)$} \label{subsubsection:standard-dual}

The vectors of the canonical basis $B_\V$ above have in principle $m-r+1$ non-zero entries. Thus the relationship between sparsity and rank is in the opposite direction than the one governing LRMC. Let us explain what we mean by that. First, the lower the rank $r$ is the denser the basis $B_\V$ is. Second, $\dim \M(r,m\times n)$ is an increasing function of $r$, so that, for minimal observation patterns of size $\dim \M(r,m\times n)$ (recall Proposition \ref{prp:matroid}), the lower the rank is the fewer entries we observe per column. Putting these two facts together, we see that for small $r$, by looking at the incomplete columns of $\pi_\Omega(X)$, it is unlikely that any of them gives sufficient information for any of the columns of $B_{\V}$, because the latter are too dense, while the former are too sparse. 

Here comes one important insight in our approach: instead of representing $\V$ via $B_\V$, we represent it via a basis $B_{\V^\perp}$ for its orthogonal complement $\V^\perp$. 

\begin{ex} \label{ex:standard-dual-*}
For $\V \in \Gr(2,6)$ the standard local coordinates for $\V^\perp \in \Gr(4,6)$ have the form $$B_{\V^\perp}=\begin{bmatrix} 
1 & 0 & 0 & 0 \\
0 & 1 & 0 & 0 \\
0 & 0 & 1 & 0 \\
0 & 0 & 0 & 1 \\
* & * & * & * \\
* & * & * & * \\
\end{bmatrix} $$
\end{ex}
 
\noindent Notice that now the sparsity level per column has become $r+1$. As far as the nature of the $*$ entries is concerned, this is expressible in terms of the Pl{\"u}cker coordinates of $\V^\perp$ by direct determinant computations, as illustrated next: 

\begin{ex} \label{ex:standard-dual}
In Example \ref{ex:standard-dual-*} let us concentrate on the value $*$ of the matrix $B_{\V^\perp}$ at location $(5,2)$. If we compute the $4 \times 4$ determinant of $B_{\V^\perp}$ defined by the row indices $1,3,4,5$, we see that $*$ is precisely equal to that determinant, that is, it is equal to $[1345]_{\V^\perp} / [1234]_{\V^\perp}$. By a similar reasoning, we see that
$$ B_{\V^\perp} \sim \begin{bmatrix} 
[1234]_{\V^\perp} & 0 & 0 & 0 \\
0 & [1234]_{\V^\perp} & 0 & 0 \\
0 & 0 & [1234]_{\V^\perp} & 0 \\
0 & 0 & 0 & [1234]_{\V^\perp} \\
-[2345]_{\V^\perp} & [1345]_{\V^\perp} & -[1245]_{\V^\perp} & [1235]_{\V^\perp} \\
-[2346]_{\V^\perp} & [1346]_{\V^\perp} & -[1246]_{\V^\perp} & [1236]_{\V^\perp} \\
\end{bmatrix}$$ where $\sim$ denotes equality up to scale.
\end{ex}

Finally, there is a 1-1 correspondence between the Pl{\"u}cker coordinates of $\V^\perp$ and those of $\V$: for any $\psi \subset [m]$ with $\#\psi=r$ the Pl{\"u}cker coordinate $[\psi]_\V$ is up to sign equal to $[[m]\setminus \psi]_{\V^\perp}$. This allows us to express the basis $B_{\V^\perp}$ of $\V^\perp$ in terms of a subset of the Pl\"ucker coordinates of $\V$: 

\begin{ex} 
Continuing with Example \ref{ex:standard-dual} we have that
\begin{align} B_{\V^\perp} \sim \begin{bmatrix} 
[56]_{\V} & 0 & 0 & 0 \\
0 & [56]_{\V} & 0 & 0 \\
0 & 0 & [56]_{\V} & 0 \\
0 & 0 & 0 & [56]_{\V} \\
-[16]_{\V} & -[26]_{\V} & -[36]_{\V} & -[46]_{\V} \\
[15]_{\V} & [25]_{\V} & [35]_{\V} & [45]_{\V} \\
\end{bmatrix} \label{eq:dual}
\end{align}
Since knowledge of $\V^\perp$ is equivalent to knowledge of $\V$, the Pl\"ucker coordinates of $\V$ that appear in \eqref{eq:dual} uniquely determine $\V$ and are uniquely determined by it on the open set of $\Gr(2,6)$ given by $[56]_{\V} \neq 0$. We call them standard dual coordinates of $\V$. 
\end{ex}

%%%%%%%
\subsubsection{Dual local coordinates on $\Gr(r,m)$ by SLMF's} \label{subsubsection:dual-local-SLMF}

With $\V=\C(X)$ the column-space of $X$, by working with a representation of $\V$ in terms of standard dual local coordinates as in \eqref{eq:dual}, it is more likely that information about the columns of $B_{\V}^\perp$, and thus about $\V$ itself, can be recovered from the columns of $\pi_\Omega(X)$. To be more precise, but still somewhat vague, what is needed is that the columns of $\Omega$ support sufficiently many and sufficiently well-distributed projections of the columns of $B_{\V}^\perp$; the importance of this will become evident inside the proof of Theorem \ref{thm:finite}. 

Still, there is an evident lack of flexibility since the sparsity pattern of $B_{\V}^\perp$ is fixed by our choice to work with the open set $[56]_{\V} \neq 0$. There is nothing special about the choice $[56]_{\V}$: one can go ahead and consider the standard dual local coordinates associated to $[ij]_{\V} \neq 0$ for any $1\le i<j\le 6$ and get another sparsity pattern for $B_{\V}^\perp$. This might be better or worse for a given $\Omega$, but equally limited once we ask that it works for all $\Omega$'s. Remarkably, these types of standard dual local coordinates are special instances of a large, much less known, class of local coordinates on $\Gr(r,m)$, induced by SLMF's (Definition \ref{dfn:SLMF} and  \cite{sturmfels1993maximal}). 
A key insight in our approach is to take this entire family of sparse bases simultaneously into account; again this will become clear inside the proof of Theorem \ref{thm:finite}. 
 
Let $\Phi = \bigcup_{j \in [m-r]} \varphi_j \times \{j\} \subset [m] \times [m-r]$ be an $(r,m)$-SLMF, where we assume that the $\varphi_j$'s are ordered. With $\varphi_{ij}$ the $i$th element of $\varphi_j$, we define an $m \times (m-r)$ matrix $B_\Phi=(b_{ij})$ over Pl{\"u}cker coordinates (viewed as polynomial variables associated to $\PP^{{m\choose r}-1}$; Example \ref{ex:Gr(2,4)}) as
\begin{align}
&b_{\varphi_{ij} j} = (-1)^{i-1}\big[ \varphi_j \setminus \{\varphi_{ij}\}\big]\, \, \, \text{for} \, \, \,  i=1,\dots,r+1 \, \, \, \, \, \, \nonumber \\
&\text{and} \, \, \, \, \, \, b_{kj } = 0 \, \, \, \text{for} \, \, \,  k \not\in \varphi_j  \label{eq:B-Phi}
\end{align} With $B_\Phi|_{\V}$ the evaluation of $B_\Phi$ at $\V$ (variable $[\psi]$ goes to Pl\"ucker coordinate $[\psi]_\V$), we have the important fact:

\begin{prp}[Theorem 4.6 in \cite{sturmfels1993maximal}, paraphrased] \label{prp:SLMF}
If $\Phi$ is an $(r,m)$-SLMF, there is a dense open set $\sV_\Phi \subset \Gr(r,m)$ such that for every $\V \in \sV_\Phi$ the $B_\Phi|_{\V}$ in \eqref{eq:B-Phi} is a basis for $\V^\perp$.
\end{prp} 

\begin{ex} \label{ex:SLMF-Phi1-Phi2-B}
The $(2,6)$-SLMF's $\Phi_1$ and $\Phi_2$ of Example \ref{ex:SLMF-Phi1-Phi2} induce the following bases $B_{\Phi_1}|_{\V}$ and $B_{\Phi_2}|_{\V}$ for $\V^\perp$, each valid for $\V \in \sV_{\Phi_1}$ and $\V \in \sV_{\Phi_2}$, respectively:
$$B_{\Phi_1}|_{\V} = \begin{bmatrix} 
[23]_{\V} & [24]_{\V} & [25]_{\V} & 0 \\
-[13]_{\V}& -[14]_{\V} & -[15]_{\V} & 0 \\
[12]_{\V} & 0 & 0 & 0 \\
0 & [12]_{\V} & 0 & [56]_{\V} \\
0 & 0 & [12]_{\V} & -[46]_{\V} \\
0& 0& 0& [45]_{\V}
\end{bmatrix} $$
$$B_{\Phi_2}|_{\V} = 
\begin{bmatrix}
0 \, \, \, \, \, \, \,& [24]_{\V} & [25]_{\V} & [35]_{\V} \\
[46]_{\V} & -[14]_{\V} & -[15]_{\V} & 0 \, \, \, \, \, \, \, \\
0 \, \, \, \, \, \, \, & 0 \, \, \, \, \, \, \, & 0 \, \, \, \, \, \, \, & -[15]_{\V} \\
-[26]_{\V}  & [12]_{\V} & 0 \, \, \, \, \, \, \, & 0 \, \, \, \, \, \, \, \\
0 \, \, \, \, \, \, \,& 0 \, \, \, \, \, \, \,& [12]_{\V} & [13]_{\V} \\
[24]_{\V} & 0 \, \, \, \, \, \, \,& 0 \, \, \, \, \, \, \, & 0 \, \, \, \, \, \, \, \\
\end{bmatrix}$$  
\end{ex}

%%%%%%%%%%%%%%%%%%
\section{LRMC via Pl\"ucker Coordinates} \label{section:MC-PC}

\myparagraph{LRMC as a problem on the Grassmannian} We recall from \S \ref{subsection:determinantal} the projection map $\pi_\Omega:\M(r,m \times n) \rightarrow \RR^\Omega$, write $X=[x_1 \cdots x_n] \in \M(r,m\times n)$ in column form and denote by $\C(X)$ the column-space of $X$. Let us begin by observing that since $x_j \in \C(X)$ for every $j \in [n]$, and since this relation is preserved under any linear projection, one can read from the incomplete matrix $\pi_{\Omega}(X)$ the geometric relations 
\begin{align}
\pi_{\omega_j}(x_j) \in \pi_{\omega_j}(\C(X)), \, \forall j \in [n] \label{eq:geometric-C}
\end{align} In LRMC we know $\omega_j$ and $\pi_{\omega_j}(x_j)$ but we know neither $x_j$ nor $\C(X)$ or its projections $\pi_{\omega_j}(\C(X))$. If the relations \eqref{eq:geometric-C} were somehow solved for the unknown $\C(X)$, and if it so happened that $\dim \pi_{\omega_j}(\C(X))=r$ for every $j \in [n]$, then unique completion of $\pi_{\Omega}(X)$ to $X$ would be immediate:

\begin{prp} \label{prp:Unique}
If $\dim \pi_{\omega_j}(\C(X)) = r, \, \forall j \in [n]$, then $X$ is the unique completion of $\pi_{\Omega}(X)$ with column space $\C(X)$, obtained from $\pi_{\Omega}(X)$ and $\C(X)$ by solving $n$ linear systems of equations.
\end{prp}

\noindent The proof of Proposition \ref{prp:Unique} is an immediate consequence of the following basic linear algebra fact: 

\begin{lem} \label{lem:vector-completion}
Let $\V$ be an $r$-dimensional subspace of $\RR^m$ and $\omega \subset [m]$ a subset of coordinates of $\RR^m$ such that $\dim \pi_\omega(\V) = r$. Let $x \in \RR^m$ be a vector such that $\pi_\omega(x) \in \pi_\omega(\V)$. Then there exists a unique vector $v \in \V$ such that $\pi_\omega(v) = \pi_\omega(x)$. 

Let $B=[b_1 \cdots b_r] \in \RR^{m \times r}$ be a basis of $\V$, $\psi \subset \omega$ such that $\#\psi=r$ and $\dim \pi_\psi(\V) = r$ and set $\pi_{\psi}(B) := [\pi_{\psi}(b_1) \cdots \pi_{\psi}(b_r)] \in \RR^{r \times r}$. Then $v = B \pi_{\psi}(B)^{-1} \pi_\psi(x)$.  
\end{lem}

Proposition \ref{prp:Unique} allows the following reformulation: 

\begin{form}[LRMC as a problem on the Grassmannian] \label{form:subspace}
Suppose that there is a linear subspace $\V$ of $\RR^m$ of dimension $r$ and $n$ points $x_j \in \V$ with $j \in [n]$. We know neither $\V$ nor the $x_j$'s. Suppose that we are given $n$ coordinate projections $\pi_{\omega_j} : \RR^m \rightarrow \RR^{\omega_j}$ and the knowledge that $\dim \pi_{\omega_j}(\V) = r$ for every $j \in [n]$. 
Then the problem is to determine $\V$ from the projected points $\pi_{\omega_j}(x_j), \, j \in [n]$.
\end{form}

In view of Formulation \ref{form:subspace}, we will head towards reducing the problem of whether $\Omega$ is generically uniquely or finitely completable (Definition \ref{dfn:gfc}) to whether there is a unique generic or finitely many generic subspaces $\V$ that agree with $\pi_{\Omega}(X)$ for a generic matrix $X$. This might be intuitively expected, but we need to be careful in formalizing it, especially to avoid a confusion with the use of $\emph{generic}$; indeed, the reader may have already noticed the again intuitively expected fact that, if we replace \emph{generic} with \emph{random}, then certainly no generic matrix agrees with $\pi_\Omega(X)$! 

Let us restrict our attention to a set of subspaces and matrices that are convenient to work with, while still being able to use them for infering properties of interest of $\Omega$:

\begin{lem} \label{lem:U-G}
Let $\sV_\Omega$ be the set of subspaces $\V \in \Gr(r,m)$ for which $\dim \pi_{\omega_j}(\V) = r, \, \forall j \in [n]$. Then $\sV_\Omega$ is a non-empty open (and thus dense) subset of $\Gr(r,m)$. 
\end{lem}

\begin{lem} \label{lem:U-M}
Let $\U_\Omega$ be the set of matrices $X \in \M(r, m \times n)$ whose column-space lies in $\sV_\Omega$. Then $\U_\Omega$ is a non-empty open (and thus dense) subset of $\M(r, m \times n)$.
\end{lem}

The key ingredient that we need for our reduction is:

\begin{lem} \label{lem:U->V}
Let $\U_\Omega'$ be any non-empty open set of $\U_\Omega$. For any $X \in \U_\Omega'$ there exists a non-empty open set $\sV_{\Omega,X} \subset \sV_\Omega$, such that there is a 1-1 correspondence $X' \mapsto \C(X')$ between 1) completions $X' \in \U_\Omega'$ of $\pi_\Omega(X)$ and 2) subspaces $\V \in \sV_{\Omega,X}$ satisfying 
\begin{align}
\pi_{\omega_j}(x_j) \in \pi_{\omega_j}(\V), \, \forall j \in [n] \label{eq:geometric}
\end{align} 
\end{lem}

Having the correspondence of Lemma \ref{lem:U->V} available, we can transfer the problem of unique or finite completability of $\Omega$ to the Grassmannian:

\begin{prp} \label{prp:guc-V}
$\Omega$ is generically uniquely completable (Definition \ref{dfn:gfc}) if and only if there is a non-empty open set $\U_\Omega'$ of $\U_\Omega$ such that for every $X \in \U_\Omega'$ the unique subspace $\V \in \sV_{\Omega,X}$ satisfying \eqref{eq:geometric} is $\C(X)$; with $\sV_{\Omega,X} \subset \Gr(r,m)$ the open set of Lemma \ref{lem:U->V}. 
\end{prp}

\begin{prp} \label{prp:gfc-V}
$\Omega$ is generically finitely completable (Definition \ref{dfn:gfc}) if and only if there is a non-empty open set $\U_\Omega'$ of $\U_\Omega$ such that for every $X \in \U_\Omega'$ there are finitely many subspaces $\V \in \sV_{\Omega,X}$ satisfying \eqref{eq:geometric}; with $\sV_{\Omega,X} \subset \Gr(r,m)$ the open set of Lemma \ref{lem:U->V}. 
\end{prp}

Now that we have precisely expressed this transfer from $\M(r, m \times n)$ to $\Gr(r,m)$, we can safely and for the sake of brevity bring back the attribute \emph{generic} into the picture, interpreted accordingly. For instance, Proposition \ref{prp:guc-V} can be informally restated as saying that $\Omega$ is generically uniquely completable if and only if for a generic $X$ there is a unique generic subspace that agrees with $\pi_\Omega(X)$. The matrix is generic in the sense of belonging to the non-empty open set $\U_\Omega' \subset \M(r, m \times n)$, while the subspace is generic in the sense of belonging to the non-empty open set $\sV_{\Omega,X} \subset \Gr(r,m)$. Since $\M(r, m\times n)$ and $\Gr(r,m)$ are irreducible, both these open sets are dense (Proposition \ref{prp:open}).

%%%%
\myparagraph{LRMC via Pl\"ucker coordinates} We now come to a key conceptual novelty in this work. We have seen that the passage from thinking in terms of matrices with regard to completing $\pi_\Omega(X)$ to thinking in terms of subspaces is the geometric relations \eqref{eq:geometric}; we now express these relations as algebraic equations in the Pl\"ucker coordinates of $\V$. This is done by the following proposition:

\begin{prp} \label{prp:sections}
Let $\varphi = \{i_1<\dots<i_{r+1}\}$ be a subset of $[m]$ and $\V \in \Gr(r,m)$ with $\dim \pi_{\varphi}(\V)=r$. Let $x \in \RR^m$. Then $\pi_\phi(x) \in \pi_\phi(V)$ if and only if 
\begin{align}
\sum_{k \in [r+1]} \, (-1)^{k-1} \, x_{i_k} \, [\varphi \setminus \{i_k\}]_{\V}=0 \label{eq:section}
\end{align}
where $x_{i_k}$ is the $i_k$th entry of $x$. That is, the vector
$$\big( [\varphi \setminus \{i_1\}]_{\V} \, \, \, \,  - [\varphi \setminus \{i_2\}]_{\V} \, \, \, \, \cdots \, \, \, \,  (-1)^r[\varphi \setminus \{i_{r+1}\}]_{\V}\big)$$ is the normal vector to the hyperplane $\pi_{\varphi}(\V) \subset \RR^{r+1}$.
\end{prp} 

\begin{ex} \label{ex:sections}
Consider the $\Omega$ of Example \ref{example:main} and note that $\omega_2=\{4,5,6\}$. Then for $X \in \M(2, 6 \times 5)$ and $\V \in \Gr(2,6)$ with $\dim \pi_{\{4,5,6\}} \big(\V\big) = 2$, we have $\pi_{\{4,5,6\}}(x_2) \in \pi_{\{4,5,6\}}(\V)$ if and and only if $x_{42} \, [56]_{\V} - x_{52} \, [46]_{\V} + x_{62} \, [45]_{\V}=0$, with $x_{ij}$ the $(i,j)$th entry and $x_2$ the second column of $X$. Also, $\pi_{\{4,5,6\}}(\V)$ is a plane in $\RR^3$ with normal vector $( [56]_{\V} \,  -[46]_{\V} \, \, \,  [45]_{\V} )$. 
\end{ex}

From Proposition \ref{prp:sections} we see that the constraints on $\C(X)$ induced by $\pi_\Omega(X)$ are encoded in the columns of $\pi_\Omega(X)$ with $\# \omega_j \ge r+1$ and have the form of \emph{hyperplane sections} of $\Gr(r,m)$. That is, 
for every $j \in [n]$ and every subset $$\varphi=\{i_1<\dots<i_{r+1}\} \subset \omega_j$$ one extracts from the relation $\pi_{\omega_j}(x_j) \in \pi_{\omega_j}(\V)$ the sub-relation $\pi_{\varphi}(x_j) \in \pi_{\varphi}(\V)$, equivalent to the linear equation
\begin{align}
\sum_{k \in [r+1]} \, (-1)^{k-1} \, x_{i_kj} \, [\varphi \setminus \{i_k\}]_{\V}=0 \label{eq:section-j}
\end{align} This equation defines a hyperplane of the projective space $\PP^{N}$ with $N={m \choose r} -1$. Inside this projective space sits the image of the Grassmannian $\Gr(r,m)$ under the Pl\"ucker embedding and $\V \in \Gr(r,m)$ agrees with $\pi_\Omega(X)$ if and only if $\V$ lies in the intersection of all such hyperplanes of the form \eqref{eq:section-j}, obtained as $\varphi$ varies across all subsets of $\omega_j$ of size $r+1$ and $j$ across all values in $[n]$. In view of Propositions \ref{prp:guc-V}-\ref{prp:gfc-V} and the remark following them, we have:

\begin{form}[LRMC via Pl\"ucker coordinates] \label{form:Plucker}
Let $X$ be a rank-$r$ matrix observed at $\Omega$. Geometrically, obtaining the rank-$r$ completions of $X$ amounts to intersecting $\Gr(r,m)$ embedded in $\PP^N$ via the Pl\"ucker embedding with all hyperplanes of the form \eqref{eq:section-j}. Algebraically, it amounts to solving a polynomial system of equations in Pl\"ucker coordinates, consisting of the linear equations in \eqref{eq:section-j} and the Pl\"ucker relations. 

A generic rank-$r$ matrix $X$ has finitely many completions if and only if the above intersection (system of polynomial equations) consists of finitely many generic points (solutions). $X$ is the unique generic completion of $X$ observed at $\Omega$ if and only if there is a unique such generic point (solution). 
\end{form}

We should emphasize that Formulation \ref{form:Plucker} plays a conceptual role in this paper; obtaining the completions computationally is beyond our present scope. It is though of great interest to know when the above polynomial system has finitely many generic solutions; this is because this happens exactly when $\Omega$ is generically finitely completable. 
 
Recall Proposition \ref{prp:matroid} and the remark preceeding it, that to understand the generically finitely completable patterns it is enough to understand such patterns $\Omega$ with size $\#\Omega = r(m+n-r)$. So let us next consider our polynomial system for such $\Omega$'s. Since the equations \eqref{eq:section-j} are linear equations, it is meaningful to ask what is the maximal number of linearly independent such equations that $\pi_\Omega(X)$ can give. This is answered by the next proposition:
 
\begin{prp} \label{prp:number-of-hyperplanes}
Suppose $\#\Omega = r(m+n-r)$. For a generic rank-$r$ matrix $X$ each $\omega_j$ contributes $\#\omega_j - r$ linearly independent equations of the form \eqref{eq:section-j}. Moreover, if $\Omega$ is generically finitely completable, then all these $\sum_{j \in [n]} (\# \omega_j -r)=\#\Omega - nr=r(m+n-r)-nr=r(m-r)$ equations must be linearly independent.
\end{prp}

\noindent Notice the interesting fact in Proposition \ref{prp:number-of-hyperplanes} that the maximal number $r(m-r)$ of linearly independent hyperplanes of the form \eqref{eq:section-j} that $\pi_\Omega(X)$ can give us, is precisely equal to the dimension of the Grassmannian $\Gr(r,m)$. It is a fact of algebraic geometry that the minimal number of linearly independent hyperplanes in $\PP^N$ that one needs to intersect $\Gr(r,m) \subset \PP^N$ with and get a finite number of points (subspaces of $\RR^m$ of dimension $r$) is equal to the dimension of $\Gr(r,m)$, that is $r(m-r)$. The same is true in linear algebra; if instead of $\Gr(r,m)$ we had a linear subspace $\mathcal{S} \subset \PP^N$ of dimension $r(m-r)$, the minimal number of linearly independent hyperplanes that we would need to intersect it with to get a finite set of points (in fact in this case we would get exactly one point in $\PP^N$) is $\dim \mathcal{S} = r(m-r)$. Moreover, this would be true as soon as none of the hyperplanes contained $\mathcal{S}$. But in algebraic geometry the picture is more complicated; we may have $r(m-r)$ linearly independent hyperplanes none of which containing $\Gr(r,m)$ and still their intersection with $\Gr(r,m)$ could be infinite. 

Instead, a more appropriate (but more involved) notion is that of a \emph{regular sequence}; see \cite{eisenbud2013commutative} for a definition or Appendix D in \cite{tsakiris2020algebraic} for an equivalent but more accessible definition. Indeed, the $r(m-r)$ linearly independent hyperplanes drop upon intersection the dimension of $\Gr(r,m)$ to zero (thus giving a finite number of points) if and only if the corresponding $r(m-r)$ linear equations are a regular sequence on $\Gr(r,m)$. On the other hand, being a regular sequence is not equivalent to $\Omega$ being generically finitely completable; the latter is a weaker condition because it is tantamount to saying that the intersection of the $r(m-r)$ hyperplanes with $\Gr(r,m)$ consists of finitely many points on a dense open set of $\Gr(r,m)$, allowing for infinitely many degenerate subspaces to still agree with the data $\pi_\Omega(X)$. We thus close this section by identifying a necessary condition on $\Omega$ for this weaker phenomenon to happen: 

\begin{definition} \label{dfn:relaxed-SLMF}
We call an $\Omega=\cup_{j \in [n]} \omega_j \times \{j\} \subset [m] \times [n]$ with $\#\Omega = r(m+n-r)$ a relaxed-$(r,r,m)$-SLMF if for every $\I \subset [m]$ with $\# \I \ge r+1$ we have that \begin{align}\sum_{j \in [\ell]} \max\big\{\# (\omega_j \cap \I) -r, 0 \big\} \le r(\#\I -r) \label{eq:relaxed-SLMF}\end{align} with equality for $\I=[m]$. 
\end{definition}

\begin{prp} \label{prp:necessary}
If $\Omega \subset [m] \times [n]$ is generically finitely completable, then $\Omega$ contains a relaxed-$(r,r,m)$-SLMF.
\end{prp}

%%%%%
\section{Comments on existing literature} \label{section:Alarcon}

The results of this paper were mostly motivated by the work of Alarcon, Boston \& Nowak \cite{pimentel2016characterization}. There, an attempt was made to characterize the generically finitely completable patterns for the special case where the matrix has $n=r(m-r)$ columns and each column is observed at exactly $r+1$ entries. In the language of the present paper, the main claim in \cite{pimentel2016characterization} (Theorem 1 there) is that for this special case of $\Omega$'s, generic finite completability is equivalent to being a relaxed-$(r,r,m)$ SLMF (Definition \ref{dfn:relaxed-SLMF} here). 

The 2016 work of Alarcon et al. \cite{pimentel2016characterization} made use of the 2015 work of Alarcon et al. \cite{pimentel2015deterministic}, where the importance of SLFM's was captured on an intuitive level independently of the 1993 work of Sturmfels \& Zelevinsky \cite{sturmfels1993maximal} who introduced SLMF's. Indeed, the claimed Theorem 1 in Alarcon et al. \cite{pimentel2015deterministic} is essentially the same as Theorem 4.6 of Sturmfels \& Zelevinsky \cite{sturmfels1993maximal} (Proposition \ref{prp:SLMF} here). Importantly, SLMF's were shown in \cite{pimentel2015deterministic} to occur with high probability in the low-rank regime under uniform sampling. Beyond the role of SLFM's, an important insight of \cite{pimentel2016characterization} was that LRMC can be cast as a problem on the Grassmannian. Indeed, a different polynomial system of equations in standard local coordinates on $\Gr(r,m)$ was derived from the geometric relations \eqref{eq:geometric} and was claimed to be a regular sequence in the generic case. The insight of Proposition \ref{prp:SLMF-V-check} here was also present in \cite{pimentel2016characterization} as Corollary 2.

The proof techniques though used in the present paper are very different from those in \cite{pimentel2016characterization,pimentel2015deterministic}. In fact, key technical arguments in \cite{pimentel2016characterization,pimentel2015deterministic} are problematic. For example, their proof of Lemma 4 in \cite{pimentel2015deterministic} is incomplete since it is impossible to prove that their polynomial (7) is non-zero without understanding the algebraic nature of the entries of their matrix $A$. This is precisely the matrix $B_\Phi|_\V$ here. Moreover, the proof strategy in \cite{pimentel2016characterization} presents a confusion between the notions of \emph{regular sequence} and \emph{algebraic independence}. For instance, the statement in their proof of their Lemma 2 that \emph{for almost every $(\Theta^*,\V^*)$ the polynomials in $\mathcal{F}$ are algebraically independent if and only if they are a regular sequence} is not correct in general and depends on the structure of $\Omega$. For the reader who is willing to trace the definitions of \emph{regular sequence} and \emph{algebraically independent}, e.g. in \cite{eisenbud2013commutative}, here is a counterexample to their statement:
\begin{ex} \label{example:Alarcon-Morteza}
The polynomials $c_1 x^2, c_2 xy \in \CC[x,y]$ are algebraically independent over $\CC$ for all non-zero values of $c_1,c_2 \in \CC$. But their common root locus in $\CC^2$ is the $y$ axis and thus the variety defined by these polynomials has dimension $1$. Moreover, there are no $c_1,c_2$ for which $c_1 x^2, c_2 xy$ is a regular sequence. 
\end{ex}
\noindent These issues were not addressed in the corrections \cite{pimentel2016corrections}.

The results of \cite{pimentel2016characterization,pimentel2015deterministic} have been used substantially in follow-up work, such as in the important paper \cite{pimentel2016information} and also in \cite{pimentel2018mixture}. Most used are corollaries (weaker versions) of Theorems 1 and 2 in \cite{pimentel2016characterization}, which are special cases of Theorems \ref{thm:finite} and \ref{thm:unique} here. Also commonly used is Theorem 1 in \cite{pimentel2015deterministic} (Proposition \ref{prp:SLMF} here) and Corollary 2 in \cite{pimentel2016characterization} (Proposition \ref{prp:SLMF-V-check} here). Finally, \cite{ashraphijuo2017fundamental,ashraphijuo2020union} considered the analogues of unique and finite completability in tensor completion, but they followed the same type of problematic arguments as in \cite{pimentel2016characterization}. 

Recently, the results presented in this paper together with an idea of Omid Amini from \emph{matroid theory}, led to a proof of the statement of Theorem 1 in \cite{pimentel2016characterization}. In the notation of the present paper this reads:

\begin{thm}[Proposition 3 in \cite{matroid-Tsakiris-23}] \label{thm:Amini-Tsakiris}
Let $\Omega \subset [m] \times [n]$ be an observation pattern with exactly $r+1$ observations per column. Suppose that $\#\Omega = r(m+n-r)$, which necessarily implies $n = r(m-r)$. Then $\Omega$ is generically finitely completable if and only if it is a relaxed-$(r,r,m)$ SLMF (Definition \ref{dfn:relaxed-SLMF}).
\end{thm}

%%%%%%%%%%%
\section{Proofs} \label{section:proofs}

In this section we give the proofs of our contributed statements. These proofs are expository, written in a style that respects the reader with no prior experience in algebraic geometry. The reader who insists on absolute rigor and completeness is referred to \cite{matroid-Tsakiris-23}. 

We follow a logical order and thus the proofs of \S \ref{section:MC-PC} precede those of \S \ref{section:main}, as the latter depend on the former.

%%%
\subsection{Proofs for \S \ref{section:MC-PC}}

\myparagraph{Proof of Lemma \ref{lem:U-G}}
First, let us fix a $j \in [n]$. For $\V \in \Gr(r,m)$ we have by Proposition \ref{prp:Gr-non-drop} that $\dim \pi_{\omega_j}(\V) = r$ if and only if there is a $\psi \subset \omega_j$ with $\#\psi=r$ such that the Pl\"ucker coordinate $[\psi]_\V$ is non-zero, that is $\V$ lies in the open set $\sV_\psi \subset \Gr(r,m)$ given by $[\psi]_\V \neq 0$. The open set $\sV_\psi$ is non-empty because one can easily find a subspace of dimension $r$ in $\RR^m$ whose projection onto the coordinates $\psi$ has dimension $r$. For every $j \in [n]$ consider all the subsets $\psi$ of size $r$ contained in $\omega_j$ and order them in any way so that $\psi_{j,k}$ denotes the $k$th such subset. If we now ask that $\dim \pi_{\omega_j}(\V) = r$ for every $j \in [n]$, this is the same as asking that $\V \in \sV_{\Omega,\K}:=\cap_{j \in [n]} \sV_{\psi_{j, k_j}}$ for some $\K=(k_1,\dots,k_n)$. Equivalently, we must have that $\V \in \cup_{\K} \sV_{\Omega,\K}$ where now the union is over all possible choices of $\K$. The set $\sV_{\Omega,\K}$ is the intersection of finitely many open non-empty sets and thus it is open. Moreover, it is non-empty by Proposition \ref{prp:open}. Also, $\sV_\Omega:= \cup_{\K} \sV_{\Omega,\K}$ is the finite union of non-empty open sets and is thus open and non-empty.

\myparagraph{Proof of Lemma \ref{lem:U-M}}
The set $\U_\Omega$ is non-empty, because one can always take a $\V \in \sV_\Omega$ and construct an $X \in \M(r, m \times n)$ by taking the $x_j$'s to be any spanning set of $\V$. To see that it is open, note that $X \in \U_\Omega$ if and only if there is an $m \times r$ column submatrix $X'$ of $X$ of rank $r$ whose Pl\"ucker embedding does not satisfy the equations that define the complement of $\sV_\Omega$. If we substitute the Pl\"ucker embedding of $X'$ into these equations, we get polynomials in the entries of $X'$ and thus polynomials in the entries of $X$. Thus each $r$ columns of $X$ give rise to an open set of $\M(r, m \times n)$ which is a subset of $\U_\Omega$. In fact $\U_\Omega$ is the union of all such open sets corresponding to all choices of $r$ columns of $X$. Since the union of open sets is open, $\U_\Omega$ is open. \qed

\myparagraph{Proof of Lemma \ref{lem:U->V}}
We define the open set $\sV_{\Omega,X}$ in several steps. Recalling the notation in the proof of Lemma \ref{lem:U-G}, if $\V \in \sV_\Omega$ there exists a $\K=\{k_1,\dots,k_n\}$ such that $\V \in \sV_{\Omega,\K}$ and so $\dim \pi_{\psi_{j,k_j}}(\V) = r$ for every $j \in [n]$. By \S \ref{subsection:local-coordinates-Gr} $\V$ is uniquely represented by a matrix $B \in \RR^{m\times r}$ with the $r \times r$ block corresponding to indices $\psi_{1,k_1}$ being the identity matrix (there is nothing special here about choosing $j=1$). Let $X \in \U_\Omega'$. By definition of $\sV_{\Omega,\K}$ the $r \times r$ matrix $\pi_{\psi_{j,k_j}}(B)$ is invertible $\forall j \in [n]$. Hence we can define the matrix
\small
$$X_\V:=\Big[B \pi_{\psi_{1,k_1}}(B)^{-1} \pi_{\psi_{1,k_1}}(x_1) \cdots B \pi_{\psi_{n,k_n}}(B)^{-1} \pi_{\psi_{n,k_n}}(x_n)\Big] $$
\normalsize
Since the columns of $X_\V$ lie in the column-space of $B$, $X_\V$ has rank at most $r$, that is $X_\V \in \M(r,m \times n)$. Now we define an open set $\sV_{\Omega,\K,X} \subset \sV_{\Omega,\K}$ to consist of those $\V \in \sV_{\Omega,\K}$ for which $X_\V \in \U_\Omega'$. Let us explain why this is an open set of $\Gr(r,m)$. First, $\U_\Omega'$ is defined by the non-simultaneous vanishing of certain polynomial equations in the entries of an $m \times n$ matrix of variables. If we take these equations and replace the variables with $X_\V$ we get rational functions in the entries of $B$, the denominators consisting of products and powers of $\det\big(\pi_{\psi_{j,k_j}}(B)\big)$'s. Clearing denominators gives a polynomial in the entries of $B$ and $X_\V \in \U_\Omega'$ if and only if one of these polynomials is non-zero. Since the entries of $B$ are local coordinates of $\Gr(r,m)$ on $\sV_{\Omega,\K}$, these polynomials correspond to polynomials in Pl\"ucker coordinates, expressed \emph{locally} at the open set $\sV_{\Omega,\K} \subset \Gr(r,m)$. Finally, our $\sV_{\Omega,X}$ is defined as $\sV_{\Omega,X}=\bigcup_{\K} \sV_{\Omega,\K,X}$ with $\K$ ranging over all possible choices. 

We now establish the claimed 1-1 correspondence. Let $X' \in \U_\Omega'$ be a completion of $\pi_\Omega(X)$. Since $\U_\Omega' \subset \U_\Omega$ we have $\C(X') \in \sV_\Omega$. Then there is a $\K$ such that $\C(X') \in \sV_{\Omega,\K}$. By Lemma \ref{lem:vector-completion} $X' = X_{\C(X')}$ where $X_{\C(X')}$ is the matrix $X_\V$ above with $\V=\C(X')$. Since $X' \in \U_\Omega'$, by definition of $\sV_{\Omega,X}$ we have $\C(X') \in \sV_{\Omega,X}$. Hence each completion $X' \in \U_\Omega'$ of $\pi_\Omega(X)$ gives a unique subspace $\C(X') \in \sV_{\Omega,X}$.

Conversely, let $\V \in \sV_{\Omega,X}$ satisfy \eqref{eq:geometric}. Since $\V \in \sV_\Omega$, by Proposition \ref{prp:Unique} $\V$ induces a unique completion $X'$ of $\pi_\Omega(X)$. By Lemma \ref{lem:vector-completion}, there is a $\K$ such that $X'$ is precisely the matrix $X_\V$ above. By definition of $\sV_{\Omega,X}$ we have that $X' \in \U_\Omega'$. 

Finally, let us see why $\sV_{\Omega,X}$ is non-empty. Take any $X \in \U_\Omega'$. Since $\U_\Omega' \subset \U_\Omega$ we have $\C(X) \in \sV_\Omega$. By the definition of $\sV_\Omega$ there is a $\K$ such that $\C(X) \in \sV_{\Omega,\K}$. By Proposition \ref{prp:Unique} we have that $X = X_{\C(X)}$. Since $X \in \U_\Omega'$ we also have $X_{\C(X)} \in \U_\Omega'$. Thus by definition of $\sV_{\Omega,X}$ we have $\C(X) \in \sV_{\Omega,X}$. \qed

\myparagraph{Proof of Proposition \ref{prp:guc-V}}
\emph{(if)} By Lemma \ref{lem:U->V}, $\Omega$ satisfies the definition of generic unique completability on $\U_\Omega'$.  

\emph{(only if)} Let $\U$ be a non-empty open set of $\M(r,m\times n)$ such that for any $X \in \U$ the only completion of $\pi_\Omega(X)$ in $\U$ is $X$. Intersect $\U$ with $\U_\Omega$ and call the result $\U_\Omega'$. By Proposition \ref{prp:open} $\U_\Omega'$ is non-empty and open. Also for any $X \in \U_\Omega'$ the only completion of $\pi_\Omega(X)$ in $\U_\Omega'$ is $X$ itself. By Lemma \ref{lem:U->V} $\C(X)$ is the unique subspace in $\sV_{\Omega,X}$ satisfying \eqref{eq:geometric}. \qed

\myparagraph{Proof of Proposition \ref{prp:gfc-V}}
\emph{(if)} By Lemma \ref{lem:U->V} there are finitely many completions of $\pi_\Omega(X)$ in $\U_\Omega'$. In view of Proposition \ref{prp:X} we have that $\Omega$ is generically finitely completable. 

\emph{(only if)} Suppose that there is a non-empty open set $\U$ of $\M(r, m \times r)$ such that for every $X \in \U$ there are finitely many completions of $\pi_\Omega(X)$ in $\U$. Intersect $\U$ with $\U_\Omega$ and call the result $\U_\Omega'$. Then for any $X \in \U_\Omega'$ there are finitely many completions of $\pi_\Omega(X)$ in $\U_\Omega'$. By Lemma \ref{lem:U->V} there are finitely many $\V \in \sV_{\Omega,X}$ that satisfy \eqref{eq:geometric}. \qed

\myparagraph{Proof of Proposition \ref{prp:sections}}
With $B \in \RR^{m \times r}$ a basis for $\V$, for every $k \in [r+1]$ we identify $[\varphi \setminus \{i_k\}]_{\V}$ with $\det \big(\pi_{\varphi \setminus \{i_k\}}(B) \big)$. Applying Laplace expansion on the first column of the matrix $[\pi_\varphi(x) \, \, \pi_\varphi(B)] \in \RR^{(r+1) \times (r+1)}$ shows that $\det \big([\pi_\varphi(x) \, \, \pi_\varphi(B)]\big)=0$ is equivalent to equation \eqref{eq:section}. Since $\pi_{\varphi}(B)$ has rank $r$, $\det \big([\pi_\varphi(x) \, \, \pi_\varphi(B)]\big)=0$ is equivalent to $\pi_\varphi(x) \in \pi_\varphi(\V)$. \qed

\myparagraph{Proof of Proposition \ref{prp:number-of-hyperplanes}}
Take $X \in \U_\Omega$, then $\C(X) \in \sV_\Omega$. By the proof of Lemma \ref{lem:U-G}, for every $j \in [n]$ there exists some $\psi_j \subset \omega_j$ with $\#\psi_j = r$ and $\dim \pi_{\psi_j} (\C(X)) = r$. By Lemma \ref{lem:vector-completion} we have that $x_j$ is the unique vector of $\C(X)$ with $\pi_{\psi_j}(x_j) \in \pi_{\psi_j}(\C(X))$. Thus, with $\C(X)$ unknown, and again by Lemma \ref{lem:vector-completion}, the relation $\pi_{\omega_j}(x_j) \in \pi_{\omega_j} (\C(X))$ is equivalent to the $\#\omega_j-r$ relations $\pi_{\psi_j \cup \{k\}}(x_j) \in \pi_{\psi_j \cup \{k\}} (\C(X)), \, \forall k \in \omega_j \setminus \psi_j$. For the second claim, we note that intersecting $\Gr(r,m)$ with fewer than $r(m-r)=\dim \Gr(r,m)$ hyperplanes gives infinitely many points, and these points stay infinite if we restrict them to any non-empty open set. \qed

\myparagraph{Proof of Proposition \ref{prp:necessary}}
By Proposition \ref{prp:matroid} $\Omega$ contains an $\Omega'$ which is generically finitely completable and $\#\Omega'=r(m+n-r)$. Suppose that $\Omega'$ violates \eqref{eq:relaxed-SLMF} for some $\I$. That is $\sum_{j \in \J} \big[\#(\omega_j' \cap \I) -r\big] > r(\#\I-r)$, where $\J$ indexes the $\omega_j$'s for which $\#(\omega_j' \cap \I)>r$. Fix a $\psi \subset \I$ with $\#\psi=r$. For generic rank-$r$ $X$, by Proposition \ref{prp:number-of-hyperplanes} each $\omega_j'$ contributes exactly $\#\omega_j'-r$ linearly independent equations of the form \eqref{eq:section-j}. There are many ways to select $\#\omega_j'-r$ linearly independent equations per $j$, but as in the proof of Proposition \ref{prp:number-of-hyperplanes}, for $j \in \J$ we can choose $\#(\omega_j' \cap \I) -r$ of them to be supported on $\I$. Thus among a total of $r(m-r)$ linearly independent equations, there are more than $r(\#\I-r)$ of them supported on $\I$; call them $\mathscr{E}$.
Now, on the open set where the Pl\"ucker coordinate $[\psi]$ is non-zero, a subspace $\V$ is uniquely represented by an $m \times r$ basis matrix $B$ whose $r \times r$ submatrix corresponding to row indices $\psi$ is the identity matrix (\S \ref{subsection:local-coordinates-Gr}). Note that there are exactly $r(m-r)$ free variables in $B$ and $r(\#\I-r)$ of them in the row-submatrix of $B$ indexed by $\I$. Now, we can get the Pl\"ucker coordinates for $\V$ by computing all $r \times r$ determinants of $B$. Substituting this expression of the Pl\"ucker coordinates into our polynomial system, gives a new polynomial system of $r(m-r)$ equations in $r(m-r)$ polynomial variables; the Pl\"ucker relations immediately vanish upon this substitution. But $\mathscr{E}$ now consists of more than $r(\#\I-r)$ equations in only $r(\#\I-r)$ variables. This contradicts the fact that $\Omega$ is generically finitely completable, since the polynomial system will have infinitely many solutions, even if restricted to a non-empty open set. \qed

%%%%%
\subsection{Proofs for \S \ref{section:main}} 

We begin by making a crucial connection between SLMF's (Definition \ref{dfn:SLMF}) and incomplete data. It is a criterion for determining membership of a point to a subspace, from memberships of a set of projections of the point to the corresponding projections of the subspace. This is interesting in its own right but is also needed in the proofs. Recall that if $\Phi$ is an $(r,m)$-SLMF, there is a dense open set $\sV_\Phi$ of $\Gr(r,m)$ on which the SLMF induces a local set of coordinates as per Proposition \ref{prp:SLMF}. We first need the following:

\begin{lem} \label{lem:Phi-non-drop}
Let $\Phi=\bigcup_{j \in [m-r]} \varphi_j \times \{j\} \subset [m] \times [m-r]$ be an $(r,m)$-SLMF. If $\V \in \sV_\Phi$ then $\dim \pi_{\varphi_j}(\V)  = r$ for every $j \in [m-r]$. 
\end{lem}
\begin{proof}
By Proposition \ref{prp:SLMF} the matrix $B_\Phi|_{\V}$ has full column rank. On the other hand, by Proposition \ref{prp:Gr-non-drop} we have that $\dim \pi_{\varphi_j}(\V)  < r$ if and only if all Pl{\"u}cker coordinates of $\V$ supported on $\varphi_j$ are zero. In that case the $j$th column of $B_\Phi|_{\V}$ would be zero, a contradiction.
\end{proof}

Here is the membership criterion:
\begin{prp} \label{prp:projections-Phi}
Let $\Phi=\bigcup_{j \in [m-r]} \varphi_j \times \{j\} \subset [m] \times [m-r]$ be an $(r,m)$-SLMF. Let $x \in \RR^m$,  $\V \in \sV_\Phi$ and suppose $\pi_{\varphi_j}(x) \in \pi_{\varphi_j}(\V), \, \forall j \in [m-r]$. Then $x \in \V$. 
\end{prp}
\begin{proof}
By Lemma \ref{lem:Phi-non-drop} we have that $\dim \pi_{\varphi_j}(\V)  = r$ for every $j \in [m-r]$. Hence by Proposition  \ref{prp:sections} the relations $\pi_{\varphi_j}(x) \in \pi_{\varphi_j}(\V)$ for every $j \in [m-r]$ can be expressed in terms of Pl\"ucker coordinates. Inspecting these relations one realizes that they are identical to saying that $x$ is orthogonal to the columns of $B_\Phi|_{\V}$. But by Proposition \ref{prp:SLMF} the columns of $B_\Phi|_{\V}$ form a basis for $\V^\perp$, hence $x \in \V$.   
\end{proof}

\myparagraph{Proof of Theorem \ref{thm:finite}}

We illustrate the proof using the $\Omega$ in Example \ref{example:main} and note that essentially the same arguments work in general. We will specify an open set $\U^*$ of $\M(2,5 \times 6)$ and construct an $X \in \U^*$ such that $\pi_{\Omega}^{-1}\big(\pi_\Omega(X)\big) \cap \U^*$ consists of $X$ alone. Then we will be done by Proposition \ref{prp:X}. By Examples \ref{ex:SLMF-Phi1-Phi2} and \ref{ex:Thm} there is a partition $\J_1=\{1,2\}, \, \J_2=\{3,4,5\}$ of $[5]$ such that the columns of $\Omega$ associated with $\J_1$ and $\J_2$ induce two $(2,6)$-SLMF's $\Phi_1$ and $\Phi_2$, respectively. Let $\sV_{\Phi_k}, \, k=1,2$ be the open sets of Proposition \ref{prp:SLMF} associated to these SLFM's. For similar reasons as in the proof of Lemma \ref{lem:U-M}, the set of $X \in \M(2,6 \times 5)$ such that $\C(X) \in \sV_{\Phi_1} \cap \sV_{\Phi_2}$ is dense and open, call it $\U$. The set of $X$'s for which $\dim \pi_{\omega_3} (\C(X)) = 2$ is also dense open, call it $\U'$. Define $\U^*=\U \cap \U'$ and $\sV^*$ the dense open set of $\Gr(2,6)$ arising as the intersection of $\sV_{\Phi_1} \cap \sV_{\Phi_2}$ with the open set of $\V$'s that satisfy $\dim \pi_{\omega_3} (\V) = 2$. Pick $\V \in \sV^*$, let $v_1,v_2$ be a basis of $\V$ and set 
$$X = [x_1\, \, x_2 \, \, x_3 \, \, x_4 \, \, x_5]:=[\underbrace{v_1\, \, v_1}_{\J_1} \, \, \underbrace{v_2 \, \, v_2 \, \, v_2}_{\J_2}]$$ Take any $X'=[x'_1\, \, x'_2 \, \, x'_3 \, \, x'_4 \, \, x'_5] \in \pi_{\Omega}^{-1}\big(\pi_\Omega(X)\big) \cap \U^*$. By hypothesis $\C(X') \in \sV^*$ and $\pi_{\omega_j}(x_j)=\pi_{\omega_j}(x'_j) \in \pi_{\omega_j}(\C(X'))$ for every $j \in [5]$. Write $\Phi_1 =\bigcup_{j \in [4]} \varphi_j^{1} \times \{j\}$ with the $\varphi_j^{1}$'s subsets of $[6]$. Then $\pi_{\omega_j}(x_j) \in \pi_{\omega_j}(\C(X'))$ for $j=1,2$ imply the relations $\pi_{\phi_k}(v_1) \in \pi_{\varphi_k^{1}}(\C(X')), \, \forall k \in [4]$. But then Proposition \ref{prp:projections-Phi} gives $v_1 \in \C(X')$. A similar argument applied with the last three  columns of $\Omega$ and $\Phi_2$ gives $v_2 \in \C(X')$. Thus necessarily $\C(X')=\C(X)$. By Lemma \ref{lem:Phi-non-drop} and the definition of $\U^*$ and $\sV^*$ we have that $\dim \pi_{\omega_j}(\C(X')) = r, \, \forall j \in [5]$. Thus Proposition \ref{prp:Unique} gives $X' = X$. \qed  

\myparagraph{Proof of Proposition \ref{prp:SLMF-V-check}}
The \emph{only if} part follows from Proposition \ref{prp:SLMF}. For the \emph{if} part, suppose there is a non-empty open set $\sV_1$ of $\Gr(r,m)$ such that for every $\V \in \sV_1$ we have $\rank (B_\Phi|_{\V}) = m-r$. Let $\sV_2$ be the open set of $\Gr(r,m)$ where none of the Pl\"ucker coordinates is zero. Set $\sV_3 = \sV_1 \cap \sV_2$ and take $\V \in \sV_3$. By Proposition \ref{prp:Gr-non-drop} each $\pi_{\varphi_j}(\V)$ is a hyperplane of $\RR^{\varphi_j}$. By Proposition \ref{prp:sections}, the vector of Pl\"ucker coordinates with $\varphi:= \varphi_j$ that appear in \eqref{eq:section} is the normal vector to that hyperplane. The $j$th column of $B_\Phi|_{\V}$ consists of these Pl\"ucker coordinates appearing at row indices $\varphi_j$ and it is orthogonal to $\V$. Hence $B_\Phi|_{\V}$ is a basis for $\V^\perp$. 

Suppose for the sake of contradiction that $\Phi$ is not $(r,m)$-SLMF. By Proposition \ref{prp:SLMF-algebraic}, some $(m-r) \times (m-r)$ determinant of $B_\Phi|_{\V}$ must be zero. But this is the same as saying that some Pl\"ucker coordinate of $\V^\perp$ is zero, say $[\psi]_{\V^\perp} = 0$ for some $\psi \subset [m]$ with $\#\psi = m-r$. But up to a sign, $[\psi]_{\V^\perp}$ is equal to $\big[[m] \setminus \psi\big]_{\V}$; contradicting the definition of $\sV_2$. \qed

\myparagraph{Proof of Theorem \ref{thm:unique}}
We illustrate the proof using the $\Omega'$ of Example \ref{ex:Thm}, which we rename $\Omega$. Let $\J_1,\J_2,\J_3, \Phi_1,\Phi_2,\Phi_3$ be as in Example \ref{ex:Thm}. We use $\bar{\cdot}$ to indicate $\bar{\cdot}$ is an object associated to the $6 \times 4$ column sub-pattern of $\Omega$ with columns indexed by $\J_1 \cup \J_2$. By Theorem \ref{thm:finite}, $\bar{\Omega}$ is generically finitely completable. Hence there is a (non-empty) open set $\bar{\U}_1 \subset \M(2, 6 \times 4)$ such that for any $Y \in \bar{\U}_1$, $\pi_{\bar{\Omega}}(Y)$ has finitely many completions. Let $\bar{\U}_2 \subset \M(2, 6 \times 4)$ be the open set defined by $Y \in \bar{U}_2$ if and only if $\C(Y) \in  \cap_{k \in [3]} \sV_{\Phi_k}$, with $\sV_{\Phi_k}$ as in Proposition \ref{prp:SLMF}. Set $\bar{\U}_3 = \bar{\U}_1 \cap \bar{\U}_2$. Now let $\U_1 \subset \M(2, 6 \times 4)$ be the open set defined by requiring $X \in \U_1$ if and only if $\bar{X} \in \bar{\U}_3$. Take an $X \in \U_1$ and let $Y$ be a completion of $\pi_{\bar{\Omega}}(\bar{X})$ in $\bar{\U}_3$; there are finitely many such $Y$'s. For $k \in [4]$ and by Lemma \ref{lem:Phi-non-drop}, the projection of $\C(Y)$ onto the coordinates $\varphi_k^3$ is a plane $\H_{Y,k}$ of $\RR^{\varphi_k^3} \cong \RR^3$ and so is the projection of $\C(X)$; call this latter $\H_{X,k}$. Now, there is a non-empty open subset $\U_2 \subset \U_1$ such that for any $X=[x_1 \cdots x_6] \in \U_2$ the projected points $\pi_{\varphi_k}(x_j)$ for $j \in \J_3$ do not lie in any $\H_{Y,k}$ for any $k$ unless $\H_{Y,k} = \H_{X,k}$; the proof of this is simple once the notions of \emph{constructible set} and \emph{upper semi-continuity of the fiber dimension} are available; for these see \cite{harris2013algebraic}. We will show that for $X \in \U_2$ the only completion of $X$ in $\U_2$ is $X$. Let $X'=[x_1' \cdots x_6']$ be such a completion. The relations $\pi_{\omega_j}(x_j')=\pi_{\omega_j}(x_j) \in \pi_{\omega_j}(\C(X))$ for $j \in \J_3$ imply that for any $k \in [4]$ we have $\pi_{\varphi_k^3}(x_j') \in \H_{X,k}$. But by the definition of $\U_2$ this is only possible if $\H_{Y,k} = \H_{X,k}$ for every $k \in [4]$. By the proof of Proposition \ref{prp:SLMF-V-check}, $\C(X')=\C(X)$. Then Proposition \ref{prp:Unique} gives  $X' = X$. \qed

\myparagraph{Proof of Theorem \ref{thm:same-SLMF}} We describe the main idea; the details can be filled following the proof of Theorem \ref{thm:unique}. Since $\Phi = \bigcup_{k \in [m-r]} \varphi_k \times \{k\}$ is supported by $r$ disjoint sub-patterns of $\Omega$, observing a generic $X$ on $\Omega$ gives $\pi_{\varphi_k}(\C(X)), \, \forall k \in [m-r]$. Each $\pi_{\varphi_k}(\C(X))$ is a hyperplane of $\RR^{r+1}$ with normal vector given by Proposition \ref{prp:sections}. By Proposition \ref{prp:SLMF}, these normal vectors give $\C(X)$. Uniqueness follows from Proposition \ref{prp:Unique}.

%%%%%%%%
\section*{Acknowledgement} 
The author is grateful to Aldo Conca for many inspiring discussions on the algebraic geometry of matrix completion. He also thanks Omid Amini, Daniel I. Bernstein, Lisa Nicklasson, Louis Theran, Nick Vannieuwenhoven and Yang Yang (ORCID 0000-0003-1200-268X) for stimulating discussions on the combinatorics of low-rank matrix completion. The author acknowledges the support of CAS Project for Young Scientists in Basic Research, Grant No. YSBR-034.

\bibliographystyle{IEEEtran}
\bibliography{MC-arxiv-Feb23}

\begin{IEEEbiographynophoto}
{Manolis C. Tsakiris} holds a PhD from Johns Hopkins University in Electrical $\&$ Computer Engineering advised by Ren\'e Vidal, and a PhD in Mathematics from the University of Genova advised by Aldo Conca. His main research interests are machine learning,  applied algebraic geometry, and commutative algebra. He is currently Tenured Associate Professor in the Academy of Mathematics and Systems Science, of the Chinese Academy of Sciences in Beijing.
\end{IEEEbiographynophoto}

\end{document}